\documentclass{article} 
\usepackage{iclr2027_conference,times}

\usepackage{amsmath,amsfonts,bm}

\def\eqref#1{equation~\ref{#1}}

\def\1{\bm{1}}

\DeclareMathAlphabet{\mathsfit}{\encodingdefault}{\sfdefault}{m}{sl}
\SetMathAlphabet{\mathsfit}{bold}{\encodingdefault}{\sfdefault}{bx}{n}

\usepackage{hyperref}
\hypersetup{hidelinks}
\usepackage{url}
\usepackage{booktabs}
\usepackage{amsfonts}
\usepackage{amssymb}
\usepackage{nicefrac}
\usepackage{microtype}
\usepackage{xcolor}
\usepackage{colortbl}
\usepackage{footmisc}
\usepackage{amsmath}
\usepackage{amsthm}
\usepackage{algorithm}
\usepackage{algpseudocode}
\usepackage{graphicx}
\usepackage{multirow}
\usepackage{subcaption}
\usepackage{wrapfig}
\usepackage{placeins}
\usepackage{enumitem}
\usepackage{comment}
\usepackage{pifont}
\newtheorem{theorem}{Theorem}

\title{FORTE: Adaptive Scoring and Exact Keyframe Selection for Long-Video Question Answering}
\author{%
\makebox[\dimexpr\textwidth-2\tabcolsep\relax][c]{%
Haifeng Huang\textsuperscript{1}\qquad
Biyin Xu\textsuperscript{1}\qquad
Chunsheng Xin\textsuperscript{1}\qquad
Yang Li\textsuperscript{1}%
}%
}

\iclrfinalcopy 
\begin{document}

\maketitle
\begingroup
\renewcommand{\thefootnote}{\arabic{footnote}}
\footnotetext[1]{Iowa State University. Correspondence to: Yang Li (\href{mailto:jerryyangli@gmail.com}{jerryyangli@gmail.com}).}
\endgroup
\setcounter{footnote}{1}

\fancyhead{}
\renewcommand{\headrulewidth}{0pt}

\begin{abstract}
Query-aware keyframe selection enables multimodal large language models (MLLMs) to process long videos using only a small set of question-relevant frames. Existing score-based methods, however, typically search within a fixed, uniformly sampled candidate pool, preventing evidence outside this pool from ever being selected. Given a limited relevance-scoring budget, the key challenge is to allocate evaluations adaptively to promising frames while continuing to explore underrepresented temporal regions. We introduce FORTE, a training-free framework that addresses this challenge through two stages: adaptive relevance scoring and global keyframe optimization. Starting from sparse, uniformly distributed observations, our efficient Gaussian-process relevance predictor estimates relevance for unscored frames, exploiting temporal locality and the approximately banded kernel structure to reduce the core computation from cubic to linear time in the number of frames for fixed bandwidth. The scoring stage then selects which frames to score next by balancing predicted relevance with temporal coverage, prioritizing promising regions while also exploring less-represented parts of the video. The optimization stage selects the final keyframes by maximizing an objective that jointly captures measured relevance and temporal coverage. We derive an exact algorithm that leverages the logarithmic coverage structure to identify the optimal subset of the scored candidate pool in time linear in the pool size, for a fixed final-frame budget. Experiments on four long-video question-answering benchmarks show that FORTE achieves the highest observed mean accuracy among the compared selectors under every tested scoring budget. Further evaluations demonstrate its consistent effectiveness across different relevance scorers and downstream MLLMs. \textit{The code will be released when the paper is published.}

\end{abstract}

\section{Introduction}
\label{sec:introduction}

Long-video understanding requires locating question-relevant evidence and integrating information across extended temporal contexts. For multimodal large language models (MLLMs), directly encoding large numbers of frames results in excessively long visual-token sequences, incurring prohibitive memory and computational costs and potentially exceeding the model's context window. Existing work mitigates this challenge through long-context modeling, visual compression, and memory mechanisms \citep{moviechat,ma-llm,longvila,longvu,video-xl}. In addition to these model-side approaches, \textit{keyframe selection} approaches the problem through another dimension by determining which video frames are highly relevant to the question and hence should be supplied to the MLLM in the first place, i.e., preserving question-relevant evidence of the video within a limited \emph{frame budget}.

Training-free, query-aware keyframe selection typically uses pretrained vision--language models, such as CLIP or BLIP \citep{clip,blip}, to estimate frame--question relevance without additional training or changes to the downstream MLLM. Existing methods commonly sample frames uniformly across the video or within temporal clips, evaluate their relevance to the question, and then determine which frames to provide to the MLLM, constrained by the frame budget, for example, by balancing relevance with temporal coverage or visual diversity \citep{aks,tcs,mdp3,adard-key,focus}. 

However, uniform sampling can miss question-relevant frames that appear between sampled timestamps, as illustrated in Figure~\ref{fig:intro-overview}(a). A straightforward way to alleviate this problem is to sample more densely, but this increases the number of frames that must be evaluated for relevance, leading to higher computation and memory costs. Therefore, an effective keyframe selection framework has to operate under another constraint, a \emph{relevance-scoring budget}, which limits the number of frames to be evaluated.


\begin{figure}[t]
    \centering
    \includegraphics[width=\linewidth]{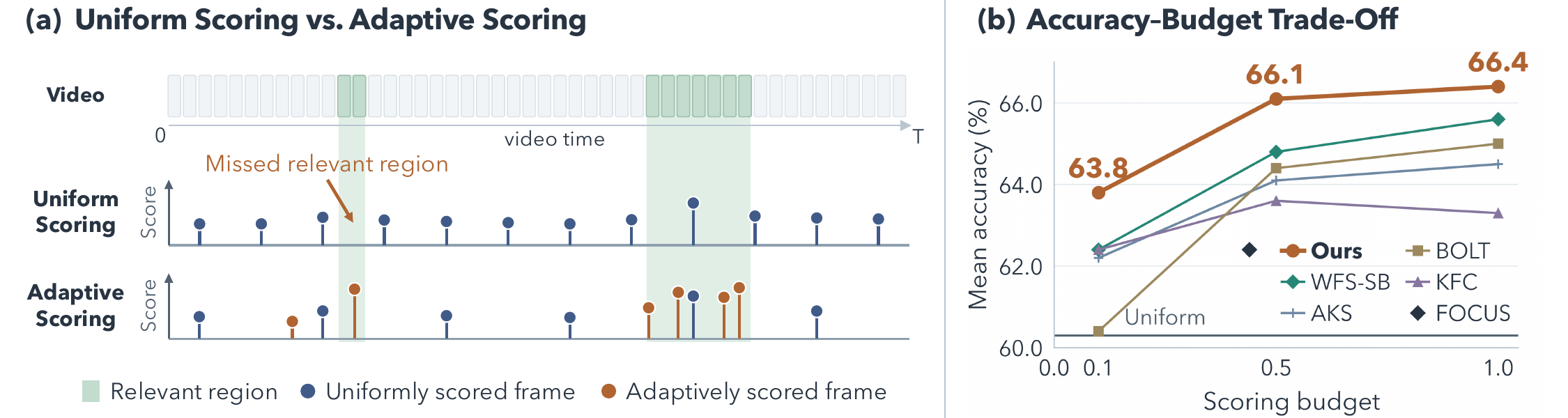}
    \caption{Budgeted keyframe selection for long-video understanding. (a) Adaptive scoring can capture relevant content missed by uniform scoring under the same evaluation budget. (b) Four-benchmark mean QA accuracy with Qwen3-VL-8B and 64 input frames. The scoring budget counts frame--query evaluations per second of video.}
    \label{fig:intro-overview}
\end{figure}


This constraint motivates us to adaptively select which frames to be scored rather than fixing their locations uniformly, to meet the relevance-scoring budget while capturing the question-relevant frames, as illustrated in Figure~\ref{fig:intro-overview}. 
Additionally, simply retaining the highest-scoring frames can concentrate the final input in a few regions, even if scoring has covered the video broadly. As such, a coherent keyframe selection framework must optimize both relevance-scoring and temporal frame coverage. 

Motivated by these two design principles, we propose \textit{FORTE}, an adaptive, training-free framework for key\underline{f}rame \underline{o}ptimization with \underline{r}elevance prediction and \underline{te}mporal coverage. Its two stages, \emph{adaptive relevance scoring} and \emph{global keyframe optimization}, are coordinated by a unified relevance--coverage objective (Figure~\ref{fig:method-overview}). The first stage is an efficient Gaussian-process (GP) relevance predictor \citep{gpml} for unscored frames, exploiting temporal locality and the resulting approximately banded kernel structure to reduce the core GP computation from cubic to linear time. The predictor is updated as new scores are observed, and its estimates are combined with temporal coverage gains to choose additional frames for evaluation. Once the relevance-scoring budget is exhausted, the second stage optimizes the relevance--coverage objective over the final frame set using measured scores. By exploiting the logarithmic gap rewards and temporal ordering, we derive an exact algorithm whose runtime is linear with regard to the candidate-pool size under a fixed final-frame budget.

Experiments on four long-video QA benchmarks show FORTE achieves the highest observed mean accuracy among compared selectors across all tested relevance-scoring and final-frame budgets. Figure~\ref{fig:intro-overview}(b) shows diminishing accuracy gains as the relevance-scoring budget increases, highlighting the importance of allocating evaluations effectively rather than simply scoring more frames.

Our contributions are summarized as follows:
\begin{itemize}[leftmargin=*]
\item We introduce FORTE, an adaptive, training-free framework that coordinates frame-level relevance scoring and final-frame selection under separate budgets. Its unified relevance--coverage objective guides feedback-driven scoring and enables efficient global optimization of the final frame set.

\item We develop an efficient GP relevance predictor for unscored frames, exploiting temporal locality and the resulting approximately banded kernel structure to reduce the core GP computation from cubic to linear time.

\item For global keyframe optimization, we derive an exact algorithm that uses a monotone upper envelope to solve the proposed objective in $\mathcal{O}(KN)$ time, where $N$ is the scored-pool size and $K$ is the final-frame budget.

\item We demonstrate the framework's effectiveness on four long-video QA benchmarks across different relevance-scoring budgets, final-frame budgets, relevance scorers, and downstream models, and examine the trade-off between scoring cost and accuracy.
\end{itemize}

\section{Related Work}
\label{sec:related-work}

\subsection{Video LLMs for Long-Video Understanding}

Video LLMs must retain useful temporal information while processing long visual sequences under finite context and computation limits. LongVILA extends the supported context through long-context training and efficient sequence parallelism \citep{longvila}. MovieChat and MA-LMM aggregate historical observations through memory mechanisms \citep{moviechat,ma-llm}, while LongVU and Video-XL compress visual representations \citep{longvu,video-xl}. While these approaches improve how video information is represented and retained, keyframe selection addresses a complementary question: which frames should be supplied to the model in the first place?

\subsection{Keyframe Selection Methods}

Training-based methods learn frame-selection policies from supervision or feedback \citep{sevila,frame-voyager,refocus}. We focus instead on training-free methods, particularly LLM-guided evidence search and lightweight score-based selection.

\noindent\textbf{LLM-guided evidence search.} These methods use language or multimodal reasoning to determine which evidence to inspect. VideoAgent coordinates iterative retrieval with vision--language tools \citep{videoagent}; VideoTree organizes captioned frames into a query-adaptive hierarchy \citep{videotree}; and A.I.R. combines lightweight candidate proposals with VLM-based analysis \citep{air}. Such workflows can exploit semantic judgments beyond frame--query similarity, but often incur substantial computational overhead from additional captioning, analysis, or reasoning calls during evidence selection.

\noindent\textbf{Score-based frame selection.} A lightweight alternative uses pretrained vision--language models, such as CLIP and BLIP \citep{clip,blip}, to score frame--query relevance. These methods typically score frames sampled uniformly throughout the video and use the resulting relevance scores to select input frames for the MLLM. Beyond simply retaining the highest-scoring frames, BOLT~\citep{bolt} samples frames according to temporal relevance mass, Q-Frame~\citep{q-frame} combines relevance-based sampling with resolution adaptation, and AdaQ~\citep{adaq} adjusts its sampling distribution based on score statistics. AKS~\citep{aks} balances relevance with temporal coverage, while TCS~\citep{tcs} combines relevance-focused clips with sparse global context. FOCUS~\citep{focus} estimates clip-level mean relevance together with a confidence bound. However, these methods do not effectively estimate frame-level relevance while jointly leveraging temporal coverage to adaptively determine which frames should be scored.

In contrast, FORTE performs adaptive frame-level scoring by updating GP predictions between acquisition batches and balancing predicted relevance against marginal temporal coverage. Moreover, its logarithmic temporal-gap objective admits exact keyframe optimization in linear time via a monotone upper-envelope algorithm.

\section{Method}
\label{sec:method}

\subsection{Problem Formulation and Overview of FORTE}

As illustrated in Figure~\ref{fig:method-overview}(a), our framework performs query-aware keyframe selection to provide a downstream MLLM with a small set of input frames for long-video question answering.

Given a video spanning $[0,T]$ and a question $q$, let $\mathcal{U}\subseteq[0,T]$ denote the timestamps of all frames in the original video, with $x_t$ denoting the frame at timestamp $t$. We use a pretrained scorer $\psi$ to evaluate frame--question relevance and denote the normalized relevance score of $x_t$ by $r_t\in[0,1]$. We distinguish two budgets: the \emph{scoring budget} $B$ limits how many frames can be evaluated by $\psi$, whereas the \emph{final frame budget} $K$ specifies how many frames are supplied to the MLLM.

The task involves two coupled decisions: which frames to score and which scored frames to retain. The first constructs a scored candidate set $\mathcal{O}\subseteq\mathcal{U}$ with $|\mathcal{O}|\leq B$; the second selects $\mathcal{S}\subseteq\mathcal{O}$ with $|\mathcal{S}|=K$. 
Candidate construction thus determines the frames available for final selection. We describe the case $|\mathcal{O}|\geq K$ below; 
Appendix~\ref{app:completion-policy} specifies how shorter candidate pools are handled.

As shown in Figure~\ref{fig:method-overview}, FORTE consists of two stages coordinated by a unified relevance--coverage objective. \emph{Adaptive relevance scoring} begins by scoring a sparse set of uniformly spaced frames. It then predicts relevance for unscored frames and balances these predictions with temporal coverage to allocate the remaining scoring budget. \emph{Global keyframe optimization} determines the final $K$ frames by balancing measured relevance with coverage of the retained subset. Coverage thus guides both the construction of the scored candidate pool and the selection of the final input. The scorer and downstream MLLM remain frozen throughout this training-free process, while $B$ and $K$ can be specified independently.


\begin{figure}[t]
    \centering
\includegraphics[width=\linewidth]{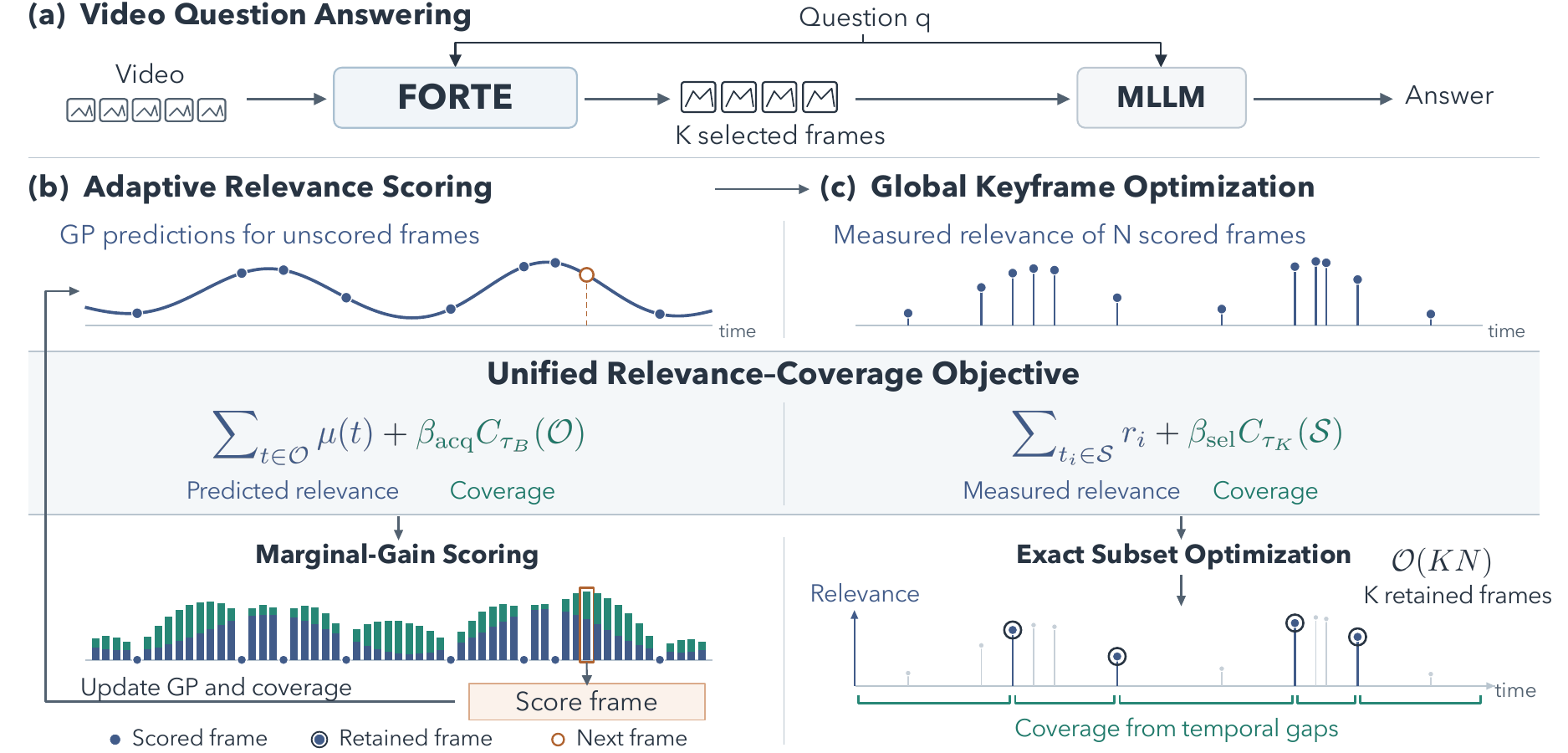}
    \caption{Overview of FORTE. (a) Under a scoring budget $B$, FORTE supplies $K$ frames and the question to a frozen MLLM. Both stages balance summed relevance with temporal coverage: (b) adaptive scoring follows marginal gains under current GP predictions; (c) global keyframe optimization uses measured scores to select the optimal $K$-frame subset from $N$ scored candidates.\vspace{-3pt}}
    \label{fig:method-overview}
\end{figure}

\subsection{Adaptive Relevance Scoring}
\label{sec:acquisition}

\noindent\textbf{Uniform initialization.} To obtain initial relevance observations without favoring any temporal region, we score a sparse set of uniformly spaced frames. By default, we allocate $B/2$ evaluations to this initialization and reserve the remaining budget for adaptive scoring.

\noindent\textbf{Predicting relevance for unscored frames.} To decide where further scoring is useful, we need relevance estimates for frames that have not yet been evaluated. Temporal continuity provides a basis for these estimates: question-relevant content can persist across successive frames, allowing measured scores to inform predictions for nearby frames. We develop an efficient Gaussian-process (GP) relevance predictor \citep{gpml} that models this dependence and exploits the temporal kernel structure to accelerate prediction. It estimates relevance from the observed scores and timestamps without additional calls to the scorer.

Let $\mathcal{O}$ denote the timestamps of the frames scored so far and $\mathbf r$ the corresponding relevance-score vector. We model temporal dependence with a radial basis function (RBF) kernel matrix $\mathbf K$, whose pairwise entries for observed timestamps $t_i,t_j\in\mathcal{O}$ are 
\vspace{-2pt}
\begin{equation}
    [\mathbf K]_{ij}=k(t_i,t_j)=\exp\!\left(-\frac{(t_i-t_j)^2}{2\ell^2}\right),
    \label{eq:rbf-kernel}
\end{equation}
where $\ell$ sets the temporal correlation scale. For an unscored frame at timestamp $t$, the vector $\mathbf k(t)$ relates $t$ to the observed timestamps through $[\mathbf k(t)]_i=k(t,t_i)$. The frame's relevance estimate is the posterior mean of GP regression:
\begin{equation}
    \mu(t)=\mu_{\mathrm{obs}}+\mathbf k(t)^\top
    (\mathbf K+\epsilon\mathbf{I})^{-1}(\mathbf r-\mu_{\mathrm{obs}}\mathbf 1).
    \label{eq:gp-mean}
\end{equation}
Here, $\mathbf 1$ is the all-ones vector. With a zero-mean GP prior, predicted relevance tends toward zero in regions with no nearby scored frames, where $\mathbf k(t)$ is small. This would make unobserved regions appear irrelevant and could discourage further scoring there. We therefore set $\mu_{\mathrm{obs}}$ to the mean of $\mathbf r$, allowing predictions in these regions to revert to the observed average. Nearby timestamps can produce very similar kernel rows, making $\mathbf K$ nearly singular. Thus, we add $\epsilon\mathbf{I}$ for numerical stability, where $\mathbf{I}$ is the identity matrix and $\epsilon>0$ is a small diagonal adjustment.

Dense GP fitting requires cubic time in the number of observed frames. This cost can be reduced by exploiting the temporal locality induced by the kernel. The RBF kernel decays rapidly with temporal separation, giving the time-ordered kernel matrix an approximately banded structure. We exploit this structure by retaining a band of local interactions and restricting prediction to neighboring observations. For a fixed neighborhood width, the banded solve requires linear time and memory in the number of observations, and each prediction requires constant arithmetic work. Unless otherwise stated, GP prediction in our experiments uses this efficient approximation. Appendix~\ref{app:gp-approximation} provides the implementation and complexity details.

\noindent\textbf{Balancing relevance and temporal coverage.}
\label{sec:coverage}
Predicted relevance helps identify promising frames, but relying on it alone may concentrate evaluations in a small part of the video. The same issue arises when selecting the final input from measured scores. We therefore design the coverage term to serve both stages, with two requirements. First, at a fixed frame count, it should favor balanced temporal spacing. Second, the coverage gained by adding a frame should diminish as nearby frames accumulate, encouraging choices from less represented regions.

For an ordered timestamp set $\mathcal{S}=\{s_1<\cdots<s_m\}\subseteq[0,T]$, let $s_0=0$ and $s_{m+1}=T$ denote the video boundaries. The temporal gaps are $g_i(\mathcal{S})=s_{i+1}-s_i$ for $i=0,\ldots,m$. We propose the following gap-based coverage objective:
\begin{equation}
C_\tau(\mathcal{S})=\sum_{i=0}^{m}\log\!\left(1+\frac{g_i(\mathcal{S})}{\tau}\right),
    \label{eq:coverage-global}
\end{equation}
where $\tau>0$ sets the reference time scale. The first property follows from the concavity of the logarithm, as formalized below.

\begin{theorem}[Uniform-spacing optimality]
\label{thm:uniform-coverage}
Among all $m$-point subsets of $[0,T]$, $C_\tau$ is maximized when the $m+1$ boundary-inclusive gaps are equal: $g_i=T/(m+1)$.
\end{theorem}

Although coverage alone is maximized by evenly spaced frames, adaptive scoring balances this preference with predicted relevance, allowing more evaluations in promising regions. For each candidate timestamp $t$, we evaluate the coverage gained by adding it to the current set, $\Delta C_\tau(t\mid\mathcal{S})=C_\tau(\mathcal{S}\cup\{t\})-C_\tau(\mathcal{S})$. This gain depends only on the gap containing $t$, so it can be computed without reevaluating coverage over the entire set. It satisfies the second property:
\begin{theorem}[Coverage gains and diminishing returns]
\label{thm:marginal-equivalence}
Given an interval $[0,T]$, a scale $\tau>0$, and a finite timestamp set $\mathcal{U}\subseteq[0,T]$, adding $t\in\mathcal{U}\setminus\mathcal{S}$ to $\mathcal{S}\subseteq\mathcal{U}$ yields a nonnegative marginal coverage gain $\Delta C_\tau(t\mid\mathcal{S})$ that depends only on the gap containing $t$. For any $\mathcal{A}\subseteq\mathcal{B}\subseteq\mathcal{U}$ and $t\in\mathcal{U}\setminus\mathcal{B}$, the gains satisfy
$\Delta C_\tau(t\mid\mathcal{A})
\geq
\Delta C_\tau(t\mid\mathcal{B}).$
Consequently, $C_\tau$ is monotone and submodular on $\mathcal{U}$.
\end{theorem}

As neighboring frames are added, further relevance evaluations in the same region receive less coverage reward. The local expression for this gain and all theorem proofs are given in Appendix~\ref{app:proofs}.

Given the current observations, we choose the next frame by its estimated marginal contribution to the relevance--coverage objective, combining predicted relevance with the additional coverage it would provide:
\begin{equation}
    t^\star={\arg\max}_{t\in\mathcal{U}\setminus\mathcal{O}}
    \left[\mu(t)+\beta_{\mathrm{acq}}\Delta C_{\tau_B}(t\mid\mathcal{O})\right],
    \label{eq:acquisition}
\end{equation}
where $\beta_{\mathrm{acq}}\geq0$ weights the coverage gain and $\tau_B=T/(B+1)$ sets the reference spacing for $B$ scored frames. We evaluate $\psi(x_{t^\star},q)$, add $t^\star$ to $\mathcal{O}$, and incorporate its measured relevance into the GP. Before choosing another frame, we update the predictions and coverage gains for the remaining unscored frames. Each new observation thus informs the next scoring decision, while the coverage reward decreases around newly scored frames. This select--score--update process, as illustrated in Figure~\ref{fig:method-overview}(b), continues until the scoring budget is exhausted.

For computational efficiency, our experiments use batched scoring, with GP updates only between batches. Within each batch, frames are chosen sequentially using fixed GP predictions, but coverage gains are updated after every choice to discourage selecting nearby frames within the same batch. Each candidate's coverage gain depends only on its neighboring selected timestamps and can be computed using $\mathcal{O}(1)$ arithmetic operations. We then score the chosen frames in parallel through batched forward passes and update the GP with their measured relevance before constructing the next batch. The batch schedule and observation updates are detailed in Appendix~\ref{app:scoring}.

\subsection{Global Keyframe Optimization}
\label{sec:final-selection}

Even when the scored candidates are distributed across the video, retaining only the highest-scoring frames can concentrate the final input in a few temporal regions. We therefore optimize the unified relevance--coverage objective over the retained subset, balancing measured relevance with coverage of the frames actually supplied to the MLLM. Let $\mathcal{O}=\{t_1<\cdots<t_N\}$ be the scored pool, and write $r_i=r_{t_i}$ for the measured relevance at $t_i$. We define the final selection objective $F$ and maximize it over $K$-frame subsets:
\begin{equation}
F(\mathcal{S}) = \sum_{t_i\in\mathcal{S}}r_i+\beta_{\mathrm{sel}}C_{\tau_K}(\mathcal{S}),\quad\quad
\mathcal{S}^{\star} = {\arg\max}_{\mathcal{S}\subseteq\mathcal{O},\,|\mathcal{S}|=K}F(\mathcal{S}),
    \label{eq:final-objective}
\end{equation}
where $\beta_{\mathrm{sel}}\geq0$ balances relevance and coverage, and $\tau_K=T/(K+1)$ is the reference spacing for the final $K$ frames. Both stages thus use an additive relevance term and the same coverage function: Equation~\ref{eq:acquisition} estimates the marginal gain from scoring another frame, while $F$ evaluates the complete retained subset using measured scores.

\noindent\textbf{Exact algorithm.} The objective $F$ decomposes into per-frame relevance and coverage rewards for consecutive temporal gaps. This structure allows us to optimize it exactly using dynamic programming, which extends optimal partial selections in temporal order. Let $D[m,i]$ be the best partial objective for $m$ selected frames ending at $t_i$, excluding the final gap to $T$. Appending $t_i$ to a subset ending at $t_j$ adds its relevance $r_i$ and the coverage reward for the gap $t_i-t_j$, giving the initialization and recurrence
\begin{equation}
    \begin{aligned}     D[1,i]&=r_i+\beta_{\mathrm{sel}}\log\!\left(1+\frac{t_i}{\tau_K}\right),
        &&1\leq i\leq N,\\
        D[m,i]&=r_i+\max_{m-1\leq j<i}
        \left[D[m-1,j]+\beta_{\mathrm{sel}}\log\!\left(1+\frac{t_i-t_j}{\tau_K}\right)\right],
        &&\begin{gathered}2\leq m\leq K,\\m\leq i\leq N.\end{gathered}
    \end{aligned}
    \label{eq:dp-transition-main}
\end{equation}
We add the final-gap reward $\beta_{\mathrm{sel}}\log(1+(T-t_i)/\tau_K)$ to each feasible $D[K,i]$ and backtrack from the best final state to recover the selected frames.

\noindent\textbf{Linear-time evaluation of each DP layer.} Scanning every predecessor in Equation~\ref{eq:dp-transition-main} would require $\mathcal{O}(KN^2)$ time. We avoid this scan by exploiting the logarithmic gap reward. For $\beta_{\mathrm{sel}}>0$, exponentiation preserves the ordering of transition values, so finding the best predecessor is equivalent to finding the largest value among the lines
\begin{equation}
    L_j(x)=\exp\!\left(\frac{D[m-1,j]}{\beta_{\mathrm{sel}}}\right)(x+\tau_K-t_j)
    \quad\text{at }x=t_i.
    \label{eq:dp-lines-main}
\end{equation}
An earlier predecessor with an equal or larger partial objective always dominates a later one, so only predecessors that establish a new maximum partial objective need to be retained. Their line slopes increase in temporal order, as do the query timestamps. We therefore maintain their upper envelope, the pointwise maximum of these lines, in a double-ended queue. Each predecessor enters and leaves the queue at most once per layer, giving $\mathcal{O}(N)$ work per layer and $\mathcal{O}(KN)$ time overall.

\begin{theorem}[Global optimality and complexity]
\label{thm:dp-optimality}
For a fixed, temporally ordered scored pool with $N\geq K$ and $\beta_{\mathrm{sel}}\geq0$, the exact algorithm returns a global maximizer of $F$ over all $K$-frame subsets in $\mathcal{O}(KN)$ time, using $\mathcal{O}(KN)$ storage including backtracking indices.
\end{theorem}

Appendix~\ref{app:proof-dp-optimality} gives the derivation, proof, and pseudocode, including evaluation without forming large exponentials and the $\beta_{\mathrm{sel}}=0$ case. After optimization, the selected frames are passed to the downstream MLLM in chronological order to generate the answer.

\begin{table}[t]
    \centering
    \caption{Long-video QA accuracy (\%) with Qwen3-VL-8B and $K=64$ input frames. Budget $\rho$ counts frame--query evaluations per second of video; FOCUS is shown at its effective budget, and ``Uniform'' requires no scoring. Avg.\ is the unweighted four-benchmark mean. Best results within each budget block are bold, including ties; green parentheses show gains over ``Uniform''.}
    \label{tab:main-results}
    \small
    \setlength{\tabcolsep}{2pt}
    \newcommand{\tablevenue}[1]{\hspace{0.4em}{\scriptsize(#1)}}
    \newlength{\tablegainwidth}
    \settowidth{\tablegainwidth}{{\footnotesize(+10.4)}}
    \newcommand{\tablegain}[1]{\makebox[\tablegainwidth][l]{{\footnotesize\textcolor{green!45!black}{(#1)}}}}
    \newcommand{\tablehead}[1]{\makebox[0pt][c]{#1}}
    \begin{tabular*}{\linewidth}{@{\extracolsep{\fill}}lc*{4}{c@{\extracolsep{0pt}\,}l@{\extracolsep{\fill}}}c@{\extracolsep{0pt}\,}l@{}}
        \toprule
        Method & Budget $\rho$ & \tablehead{LongVB} & & \tablehead{Video-MME} & & \tablehead{LVBench} & & \tablehead{MLVU} & & \tablehead{Avg.} & \\
        \midrule
        Uniform & 0 & 62.4 & & 67.7 & & 41.0 & & 70.3 & & 60.3 & \\
        \midrule
        Top-K & 0.1 & 63.4 & & 68.7 & & 46.6 & & 70.8 & & 62.4 & \\
        AKS\tablevenue{CVPR~'25} & 0.1 & 63.2 & & 69.1 & & 45.6 & & 71.0 & & 62.2 & \\
        BOLT\tablevenue{CVPR~'25} & 0.1 & 59.3 & & 63.4 & & 46.7 & & 71.9 & & 60.4 & \\
        Q-Frame\tablevenue{ICCV~'25} & 0.1 & 62.9 & & 69.1 & & 42.9 & & 70.9 & & 61.5 & \\
        AdaQ\tablevenue{arXiv~'26} & 0.1 & 63.4 & & 68.9 & & 46.7 & & 70.8 & & 62.4 & \\
        WFS-SB\tablevenue{CVPR~'26} & 0.1 & 63.2 & & 68.3 & & 46.9 & & 71.3 & & 62.4 & \\
        KFC\tablevenue{ICLR~'26} & 0.1 & 62.5 & & 68.2 & & 47.6 & & 71.4 & & 62.4 & \\
        \textbf{Ours} & 0.1 & \textbf{64.9} & \tablegain{+2.5} & \textbf{69.2} & \tablegain{+1.5} & \textbf{49.0} & \tablegain{+8.0} & \textbf{72.0} & \tablegain{+1.7} & \textbf{63.8} & \tablegain{+3.5} \\
        \midrule
        Top-K & 0.5 & 65.7 & & 68.1 & & 50.6 & & 73.8 & & 64.6 & \\
        AKS\tablevenue{CVPR~'25} & 0.5 & 65.1 & & 69.0 & & 47.8 & & 74.5 & & 64.1 & \\
        BOLT\tablevenue{CVPR~'25} & 0.5 & 63.9 & & 69.4 & & 49.2 & & 75.1 & & 64.4 & \\
        Q-Frame\tablevenue{ICCV~'25} & 0.5 & 64.4 & & 67.6 & & 43.5 & & 72.2 & & 61.9 & \\
        AdaQ\tablevenue{arXiv~'26} & 0.5 & 66.2 & & 68.3 & & 50.2 & & 73.7 & & 64.6 & \\
        WFS-SB\tablevenue{CVPR~'26} & 0.5 & 65.5 & & 68.4 & & 50.6 & & 74.7 & & 64.8 & \\
        FOCUS\tablevenue{ICLR~'26} & $\sim\!0.44$ & 64.1 & & 65.3 & & 48.7 & & 71.3 & & 62.4 & \\
        KFC\tablevenue{ICLR~'26} & 0.5 & 65.5 & & 67.8 & & 46.4 & & 74.9 & & 63.6 & \\
        \textbf{Ours} & 0.5 & \textbf{68.1} & \tablegain{+5.7} & \textbf{69.9} & \tablegain{+2.2} & \textbf{50.7} & \tablegain{+9.7} & \textbf{75.7} & \tablegain{+5.4} & \textbf{66.1} & \tablegain{+5.8} \\
        \midrule
        Top-K & 1.0 & 65.9 & & 67.5 & & 51.2 & & 74.1 & & 64.7 & \\
        AKS\tablevenue{CVPR~'25} & 1.0 & 66.2 & & 69.0 & & 49.1 & & 73.6 & & 64.5 & \\
        BOLT\tablevenue{CVPR~'25} & 1.0 & 65.6 & & 69.6 & & 49.5 & & 75.1 & & 65.0 & \\
        Q-Frame\tablevenue{ICCV~'25} & 1.0 & 64.4 & & 68.2 & & 43.9 & & 72.3 & & 62.2 & \\
        AdaQ\tablevenue{arXiv~'26} & 1.0 & 66.6 & & 67.8 & & 50.9 & & 74.3 & & 64.9 & \\
        WFS-SB\tablevenue{CVPR~'26} & 1.0 & 67.6 & & 68.5 & & 51.0 & & 75.5 & & 65.6 & \\
        KFC\tablevenue{ICLR~'26} & 1.0 & 66.0 & & 67.8 & & 45.4 & & 73.8 & & 63.3 & \\
        \textbf{Ours} & 1.0 & \textbf{68.5} & \tablegain{+6.1} & \textbf{69.9} & \tablegain{+2.2} & \textbf{51.4} & \tablegain{+10.4} & \textbf{75.8} & \tablegain{+5.5} & \textbf{66.4} & \tablegain{+6.1} \\
        \bottomrule
    \end{tabular*}
\end{table}

\section{Experiments}
\label{sec:experiments}

\subsection{Experimental Setup}

\noindent\textbf{Benchmarks and metric.} We evaluate on four complementary long-video QA benchmarks. LongVideoBench \citep{longvideobench} emphasizes evidence retrieval and referring reasoning; Video-MME \citep{videomme} spans diverse domains and video durations; LVBench \citep{lvbench} focuses on information integration over hour-scale videos; and MLVU \citep{mlvu} covers multiple understanding tasks and video genres. Together, they test both localized evidence retrieval and reasoning across extended temporal contexts. We report multiple-choice accuracy (\%) over all evaluated questions in each benchmark. \emph{Avg.} denotes the unweighted four-benchmark mean, also used in the ablation and generalization studies.

\noindent\textbf{Models and relevance scorers.} We evaluate four representative video LLMs: LLaVA-OneVision (LLaVA-OV) \citep{llava-onevision}, LLaVA-Video \citep{llava-video}, Qwen3-VL \citep{qwen3vl}, and InternVL3.5 \citep{internvl35}. Relevance scorers include BLIP \citep{blip}, BLIP-2 \citep{blip2}, and CLIP \citep{clip}. Unless otherwise stated, we use Qwen3-VL-8B-Instruct as the downstream model and BLIP as the relevance scorer in our experiments. We use greedy decoding for downstream answer generation, without stochastic token sampling.

\noindent\textbf{Scoring and frame budgets.} For a $T$-second video, we allocate $B\approx\rho T$ frame--query evaluations, where $\rho$ denotes the scoring allowance per second of video. The main comparison uses $\rho\in\{0.1,0.5,1.0\}$ and a final frame budget of $K=64$ for the MLLM. Baseline scoring protocols and budget accounting are detailed in Appendix~\ref{app:baseline-implementations}.

\subsection{Main Results}

\noindent\textbf{Baselines.} For controlled comparisons, we reimplement baseline frame-selection rules within a common evaluation pipeline, using the same downstream MLLM, relevance scorer, and final-frame budget. The Uniform baseline samples $K$ evenly spaced frames without relevance scoring. Top-K retains the $K$ highest-scoring observed frames. AKS~\citep{aks} recursively balances query relevance and temporal coverage. BOLT~\citep{bolt} uses deterministic inverse-transform sampling of normalized relevance scores. Q-Frame~\citep{q-frame} applies Gumbel-Max sampling to query--frame relevance. AdaQ~\citep{adaq} adapts its sampling distribution to score variance and samples from a top-$p$ relevance set. WFS-SB~\citep{wfssb} detects semantic boundaries in the relevance trajectory and selects diverse frames within the resulting segments. KFC, the frame-selection component of Nar-KFC~\citep{narkfc}, jointly optimizes query relevance and visual diversity. FOCUS~\citep{focus} uses clip-level relevance estimates to guide frame selection. Its official scoring schedule yields an effective budget of approximately $0.44$ frame--query evaluations per second of video, so we report it alongside the $\rho=0.5$ results. Appendix~\ref{app:focus-matched-budget} specifies an additional comparison matching FORTE to FOCUS's scoring count for each question.

\noindent\textbf{Analysis.} Table~\ref{tab:main-results} compares selectors at scoring budgets $\rho\in\{0.1,0.5,1.0\}$. Our method achieves the highest observed accuracy on all four benchmarks at every budget. Top-K becomes increasingly competitive as more frames are scored, outperforming several more complex selectors at higher budgets. However, the benefits of additional scoring diminish:
Increasing $\rho$ from $0.5$ to $1.0$ brings only small gains in average accuracy despite roughly doubling the number of scorer evaluations. These results suggest that $\rho=0.5$ offers a practical balance between accuracy and scoring cost.

\subsection{Ablation and Generalization Studies}

Motivated by the diminishing returns beyond $\rho=0.5$, we use this scoring budget and $K=64$ input frames for all ablation and generalization studies unless otherwise stated.

\noindent\textbf{Adaptive relevance scoring and global keyframe optimization.} Table~\ref{tab:component-ablation} compares variants of the two stages. In Stage 1, ``Full'' uses both terms in Eq.~\ref{eq:acquisition}, while ``w/o'' indicates removal of the specified term. 
With Stage 2 fixed to global keyframe optimization, ``Full''
outperforms uniform scoring and either single-term variant, supporting the combined use of predicted relevance and coverage. In Stage 2, global keyframe optimization (``Global opt.'') maximizes the joint relevance--coverage objective in Eq.~\ref{eq:final-objective}, whereas Top-K uses relevance alone. Global keyframe optimization improves on Top-K for
both uniformly and adaptively scored pools, showing that coverage also matters 
for final frame selection.
Appendix~\ref{app:additional-experiments} provides per-benchmark ablations and paired uncertainty estimates.

\begin{table}[t]
    \centering
    \newcommand{\pairedtablehead}[1]{\makebox[0pt][c]{#1}}
    \begin{minipage}[t]{0.41\linewidth}
        \centering
        \caption{Ablation of the two stages.}
        \label{tab:component-ablation}
    \end{minipage}\hfill
    \begin{minipage}[t]{0.57\linewidth}
        \centering
        \caption{Ablation of the final frame budget.}
        \label{tab:frame-budget}
    \end{minipage}
    \par\nointerlineskip
    \begin{minipage}[c]{0.41\linewidth}
        \centering
        \fontsize{8}{10}\selectfont
        \setlength{\tabcolsep}{2pt}
        \begin{tabular*}{\linewidth}{@{\extracolsep{\fill}}llc@{}}
            \toprule
            Stage 1 & Stage 2 & Avg. \\
            \midrule
            Uniform scoring & Top-K & 64.6 \\
            Uniform scoring & Global opt. & 65.2 \\
            w/o relevance & Global opt. & 65.2 \\
            w/o coverage & Global opt. & 64.2 \\
            Full & Top-K & 64.9 \\
            Full & Global opt. & \textbf{66.1} \\
            \bottomrule
        \end{tabular*}
    \end{minipage}\hfill
    \begin{minipage}[c]{0.57\linewidth}
        \centering
        \fontsize{8}{10}\selectfont
        \setlength{\tabcolsep}{2pt}
        \begin{tabular*}{\linewidth}{@{\extracolsep{\fill}}ccccccc@{\hspace{2pt}}}
            \toprule
            $K$ & \pairedtablehead{Uniform} & \pairedtablehead{Top-K} & \pairedtablehead{AKS} & \pairedtablehead{WFS-SB} & \pairedtablehead{KFC} & \pairedtablehead{\textbf{Ours}} \\
            \midrule
            8 & 51.4 & 55.9 & 56.2 & 57.3 & 56.3 & \textbf{58.3} \\
            16 & 54.4 & 58.4 & 58.3 & 59.6 & 59.1 & \textbf{60.5} \\
            32 & 58.2 & 60.8 & 60.7 & 62.2 & 60.9 & \textbf{63.6} \\
            64 & 60.3 & 64.6 & 64.1 & 64.8 & 63.6 & \textbf{66.1} \\
            \bottomrule
        \end{tabular*}
    \end{minipage}
\end{table}

\begin{table}[!tbp]
    \centering
    \newcommand{\pairedtablehead}[1]{\makebox[0pt][c]{#1}}
    \begin{minipage}[t]{0.41\linewidth}
        \centering
        \caption{Generalization across scorers.}
        \label{tab:scorer-ablation}
    \end{minipage}\hfill
    \begin{minipage}[t]{0.57\linewidth}
        \centering
     \caption{Generalization across downstream MLLMs.}
        \label{tab:backbone-transfer}
    \end{minipage}
    \par\nointerlineskip
    \begin{minipage}[c]{0.41\linewidth}
        \centering
        \fontsize{8}{10}\selectfont
        \setlength{\tabcolsep}{2pt}
        \begin{tabular*}{\linewidth}{@{\extracolsep{\fill}}lccccc@{\hspace{2pt}}}
            \toprule
            Scorer & \pairedtablehead{Top-K} & \pairedtablehead{AKS} & \pairedtablehead{WFS-SB} & \pairedtablehead{KFC} & \pairedtablehead{\textbf{Ours}} \\
            \midrule
            BLIP & 64.6 & 64.1 & 64.8 & 63.6 & \textbf{66.1} \\
            CLIP & 64.4 & 62.7 & 63.9 & 63.8 & \textbf{65.0} \\
            BLIP-2 & 64.4 & 62.4 & 64.3 & 64.1 & \textbf{65.1} \\
            \bottomrule
        \end{tabular*}
    \end{minipage}\hfill
    \begin{minipage}[c]{0.57\linewidth}
        \centering
        \fontsize{8}{10}\selectfont
        \setlength{\tabcolsep}{2pt}
        \begin{tabular*}{\linewidth}{@{\extracolsep{\fill}}lcccccc@{\hspace{2pt}}}
            \toprule
            MLLM & \pairedtablehead{Uniform} & \pairedtablehead{Top-K} & \pairedtablehead{AKS} & \pairedtablehead{WFS-SB} & \pairedtablehead{KFC} & \pairedtablehead{\textbf{Ours}} \\
            \midrule
            Qwen3-VL-8B & 60.3 & 64.6 & 64.1 & 64.8 & 63.6 & \textbf{66.1} \\
            LLaVA-OV-7B & 55.6 & 58.3 & 58.0 & 58.3 & 58.2 & \textbf{59.3} \\
            LLaVA-Video-7B & 59.4 & 62.0 & 61.3 & 61.8 & 61.0 & \textbf{62.6} \\
            InternVL3.5-8B & 60.7 & 62.8 & 62.7 & 62.9 & 62.5 & \textbf{63.8} \\
            \bottomrule
        \end{tabular*}
    \end{minipage}\vspace{-2pt}
\end{table}

\noindent\textbf{Final frame budget.} 
Table~\ref{tab:frame-budget} compares Uniform, Top-K, AKS, WFS-SB, KFC, and our method across final frame budgets. All six improve with more input frames, but ours achieves the highest observed mean accuracy at every tested $K$. With 16 frames, it reaches the same reported mean accuracy as Uniform with 64, suggesting that question-aware selection can preserve accuracy with substantially fewer input frames.

\noindent\textbf{Parameter sensitivity.} We use a single shared configuration of heuristic defaults without development-set tuning (Appendix~\ref{app:implementation}) and keep it fixed across all main experiments. To assess sensitivity, we vary the GP length scale, both coverage weights, the uniform initialization fraction, and the number of adaptive scoring batches one at a time, reporting the unweighted mean accuracy across the four evaluation benchmarks (Appendix~\ref{app:sensitivity}).

\noindent\textbf{Relevance scorers.} We replace BLIP with CLIP \citep{clip} and BLIP-2 \citep{blip2}, keeping selector hyperparameters fixed. Our method outperforms Top-K, AKS, WFS-SB, and KFC in mean accuracy with all three scorers (Table~\ref{tab:scorer-ablation}), showing that its gains are not specific to BLIP.

\noindent\textbf{Downstream MLLMs.} We evaluate the same frame selections with LLaVA-OneVision-7B \citep{llava-onevision}, LLaVA-Video-7B \citep{llava-video}, and InternVL3.5-8B \citep{internvl35}. While the strongest baseline varies across models, our method achieves the highest observed mean accuracy among the compared selectors on each of the four MLLMs (Table~\ref{tab:backbone-transfer}).

\noindent\textbf{Model scales.} With $K=64$ input frames, accuracy increases with Qwen3-VL model size from 2B to 32B parameters. Our method remains ahead of AKS, WFS-SB, KFC, and FOCUS at every tested scale (Figure~\ref{fig:model-scale}), showing that its advantage persists with larger downstream MLLMs.

\begin{figure}[t]
    \centering
    \begin{minipage}[t]{0.48\linewidth}
        \centering
        \includegraphics[width=\linewidth]{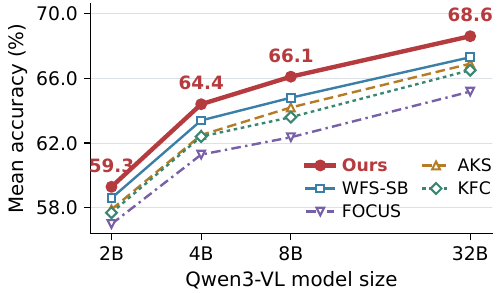}
        \caption{Generalization across model scales. Points show four-benchmark mean accuracy.}
        \label{fig:model-scale}
    \end{minipage}\hfill
    \begin{minipage}[t]{0.48\linewidth}
        \centering
        \includegraphics[width=\linewidth]{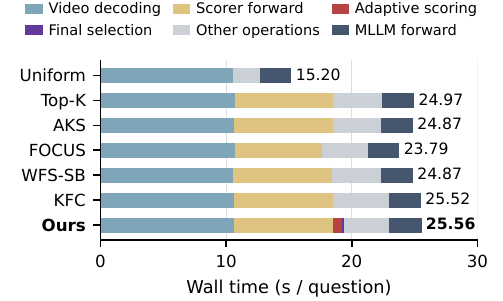}
        \caption{Wall-time decomposition of the entire pipeline for long-video question answering.}
        \label{fig:online-core-cost}
    \end{minipage}
\end{figure}

\subsection{Computational Cost}
\label{sec:computational-cost}

We evaluate runtime and peak GPU memory on 100 videos with three repeats per method, using Qwen3-VL-8B, BLIP, and $K=64$ on an A100-80GB GPU. Figure~\ref{fig:online-core-cost} reports the end-to-end breakdown at the default scoring budget. Video decoding and scorer inference account for approximately $72\%$ of FORTE's runtime, whereas adaptive scoring and final optimization together take only $0.88$ seconds ($3.4\%$). Its total latency is $25.56$ seconds per question, close to Top-K, AKS, WFS-SB, and KFC ($24.87$--$25.52$ seconds). Thus, the gains from adaptive scoring and optimization come with modest additional overhead relative to other score-based selectors. Our goal is not to reduce end-to-end cost relative to uniform sampling, which requires no relevance scoring, but to achieve higher accuracy under a fixed relevance-scoring budget. All scored methods reach the same selection-stage GPU memory peak of $2.52$\,GiB, compared with $23.46$\,GiB for MLLM answering. Appendix~\ref{app:online-cost} provides the measurement protocol and full numerical results.

\section{Conclusion}
\label{sec:conclusion}

We present FORTE, a training-free framework for long-video QA that decides which frames to score and which to retain under separate budgets. Both decisions balance relevance with temporal coverage, using efficient GP predictions during scoring and measured relevance during final selection. By exploiting the objective's logarithmic gap structure, our exact algorithm finds the globally optimal subset of the scored pool in $\mathcal{O}(KN)$ time. Across four benchmarks, FORTE achieves the highest observed mean accuracy among the compared methods under the tested budgets, consistently generalizing across the evaluated relevance scorers and downstream MLLMs.

\FloatBarrier

\bibliography{main}
\bibliographystyle{iclr2027_conference}

\clearpage
\appendix
This appendix complements the main manuscript with theoretical proofs, implementation details, and additional quantitative and qualitative analyses. It is organized as follows:
\begin{enumerate}
    \item Appendix~\ref{app:proofs} proves the temporal coverage properties and the optimality and complexity guarantees for the exact algorithm used in global keyframe optimization.
    \item Appendix~\ref{app:scoring} describes the efficient GP approximation and the batched implementation of adaptive relevance scoring.
    \item Appendix~\ref{app:implementation} specifies FORTE's implementation settings, the common evaluation protocol, and the baseline implementations.
    \item Appendix~\ref{app:additional-experiments} examines the contributions of FORTE's components and design choices, quantifies uncertainty in accuracy differences, and evaluates numerical stability, computational cost, and performance across video durations and scoring-budget regimes.
    \item Appendix~\ref{app:related-work} extends the discussion of long-video modeling and keyframe selection, clarifying how FORTE relates to existing approaches.
    \item Appendix~\ref{app:qualitative} compares scored candidates, selected frames, and model answers through qualitative examples of evidence acquisition, final retention, and a shared failure case.
\end{enumerate}

\section{Coverage Objective and Selection Guarantees}
\label{app:proofs}

The coverage objective in Section~\ref{sec:method} guides both adaptive relevance scoring and global keyframe optimization. We first prove its two key properties: it favors balanced temporal spacing at a fixed frame count, and the gain from adding a frame decreases as nearby frames accumulate. We then derive the exact algorithm and prove that it maximizes the final relevance--coverage objective in $\mathcal{O}(KN)$ time. Throughout, we use the main manuscript's interval $[0,T]$ and represent frame sets by their timestamps. The coverage scale $\tau>0$ remains fixed as the selected set changes.

\subsection{Temporal Gaps and Uniform-Spacing Optimality}
\label{app:proof-uniform-coverage}

Let $\mathcal{S}=\{s_1<\cdots<s_m\}\subseteq[0,T]$ contain the selected timestamps. Following Section~\ref{sec:coverage}, set $s_0=0$ and $s_{m+1}=T$. The temporal gaps are
\begin{equation}
    g_i(\mathcal{S})=s_{i+1}-s_i,\qquad i=0,\ldots,m.
    \label{eq:gap-definition}
\end{equation}
These gaps include the two boundary gaps, so they are nonnegative and sum to $T$. For brevity, write $f_\tau(g)=\log(1+g/\tau)$. The coverage objective is then
\begin{equation}
    C_\tau(\mathcal{S})=\sum_{i=0}^{m}f_\tau\!\left(g_i(\mathcal{S})\right).
    \label{eq:coverage-expanded}
\end{equation}
For a fixed number of selected frames, the total gap length is always $T$; only its distribution among the gaps changes. The concavity of $f_\tau$ makes balanced gaps preferable to a mixture of very short and very long gaps. For the empty set, there is one gap of length $T$, giving $C_\tau(\varnothing)=f_\tau(T)$.

\begin{proof}[Proof of Theorem~\ref{thm:uniform-coverage}]
The function $f_\tau$ is strictly concave for $g\geq0$, since $f_\tau''(g)=-(\tau+g)^{-2}<0$. Applying Jensen's inequality to the $m+1$ gaps gives
\begin{equation}
    \frac{C_\tau(\mathcal{S})}{m+1}
    =\frac{1}{m+1}\sum_{i=0}^{m}f_\tau\!\left(g_i(\mathcal{S})\right)
    \leq f_\tau\!\left(\frac{T}{m+1}\right).
\end{equation}
The bound depends only on $T$, $m$, and $\tau$, not on the selected timestamps. By strict concavity, equality holds if and only if all gaps equal $T/(m+1)$. The timestamps $s_i=iT/(m+1)$ for $i=1,\ldots,m$ attain this bound, proving Theorem~\ref{thm:uniform-coverage}.
\end{proof}

The theorem characterizes coverage alone when any timestamp can be selected. In practice, selection is restricted to the scored candidates and also accounts for relevance, so the optimal subset need not be uniformly spaced.

\subsection{Marginal Coverage and Diminishing Returns}
\label{app:proof-marginal-equivalence}

The marginal coverage gain measures the additional temporal coverage contributed by one frame. For a frame at timestamp $t\notin\mathcal{S}$, adding $t$ replaces one temporal gap by two smaller gaps; all other gaps remain unchanged. Let $L$ and $R$ be the nearest selected timestamps to the left and right of $t$, using $0$ or $T$ when the corresponding neighbor is absent. The marginal gain is therefore
\begin{equation}
    \Delta C_\tau(t\mid\mathcal{S})
    =f_\tau(t-L)+f_\tau(R-t)-f_\tau(R-L).
    \label{eq:marginal-coverage}
\end{equation}

\begin{proof}[Proof of Theorem~\ref{thm:marginal-equivalence}]
The original gap has length $R-L$. Inserting $t$ replaces its reward $f_\tau(R-L)$ by $f_\tau(t-L)+f_\tau(R-t)$, establishing Equation~\ref{eq:marginal-coverage} and its dependence only on the containing gap. Let $a=t-L\geq0$ and $b=R-t\geq0$ be the distances to the two neighbors. Writing the marginal gain as $h_\tau(a,b)$ gives
\begin{equation}
    h_\tau(a,b)
    =\log\!\left(\frac{(\tau+a)(\tau+b)}{\tau(\tau+a+b)}\right)
    =\log\!\left(1+\frac{ab}{\tau(\tau+a+b)}\right)\geq0.
\end{equation}
Thus, adding a frame cannot decrease coverage, establishing monotonicity.

To prove diminishing returns, we examine how this gain changes as the neighboring selected timestamps move closer to $t$. Its partial derivatives satisfy
\begin{equation}
    \frac{\partial h_\tau}{\partial a}
    =\frac{1}{\tau+a}-\frac{1}{\tau+a+b}\geq0,
    \qquad
    \frac{\partial h_\tau}{\partial b}
    =\frac{1}{\tau+b}-\frac{1}{\tau+a+b}\geq0.
\end{equation}
Hence, the gain is nondecreasing in each distance. Now take $\mathcal{A}\subseteq\mathcal{B}\subseteq\mathcal{U}$ and $t\in\mathcal{U}\setminus\mathcal{B}$. Because $\mathcal{B}$ contains every timestamp in $\mathcal{A}$, its nearest neighbors of $t$ can only be closer, or remain the same. Both $a$ and $b$ therefore weakly decrease when passing from $\mathcal{A}$ to $\mathcal{B}$, so $\Delta C_\tau(t\mid\mathcal{A})\geq\Delta C_\tau(t\mid\mathcal{B})$. This is precisely the diminishing-returns condition for submodularity on the finite timestamp set $\mathcal{U}$.
\end{proof}

\noindent\textbf{Relation to submodular maximization.}
Theorem~\ref{thm:marginal-equivalence} places our temporal coverage function within the class of
monotone submodular set functions. Since the relevance term is modular,
the final objective in Eq.~\ref{eq:final-objective} is also monotone submodular for
nonnegative relevance scores and $\beta_{\mathrm{sel}} \geq 0$.
Consequently, generic algorithms for cardinality-constrained
submodular maximization, such as marginal-gain greedy selection, could
be applied to this objective.

Our setting, however, has additional structure that is not used by a
generic submodular optimizer: candidate frames are totally ordered in
time, and the coverage term decomposes over consecutive temporal gaps.
For the logarithmic gap reward used here, this structure permits exact
optimization rather than approximate greedy maximization. Appendix~\ref{app:proof-dp-optimality}
exploits this property to derive the $O(KN)$ algorithm used by FORTE.

\subsection{Exact Algorithm}
\label{app:proof-dp-optimality}

We optimize $F$ over a fixed scored pool $\mathcal{O}=\{t_1<\cdots<t_N\}\subseteq[0,T]$, with $T>0$ and $1\leq K\leq N$. For brevity, write $\beta=\beta_{\mathrm{sel}}$ and $\tau=\tau_K=T/(K+1)$, keeping both fixed throughout the optimization. We first establish the recurrence, then show how the logarithmic gap reward reduces its evaluation to $\mathcal{O}(N)$ work per layer.

\noindent\textbf{Dynamic-programming recurrence.} Let $D[m,i]$ be the best partial objective for $m$ selected frames ending at $t_i$, excluding the final gap to $T$. Appending $t_i$ to a selection ending at $t_j$ adds only $r_i+\beta f_\tau(t_i-t_j)$, giving
\begin{align}
    D[1,i]&=r_i+\beta f_\tau(t_i), &&1\leq i\leq N,
    \label{eq:dp-base}\\
    D[m,i]&=r_i+\max_{m-1\leq j<i}\left[D[m-1,j]+\beta f_\tau(t_i-t_j)\right],
    &&\begin{gathered}2\leq m\leq K,\\m\leq i\leq N.\end{gathered}
    \label{eq:dp-recurrence}
\end{align}
After the last layer, we restore the right-boundary gap:
\begin{equation}
    F^\star=\max_{K\leq i\leq N}\left[D[K,i]+\beta f_\tau(T-t_i)\right].
    \label{eq:dp-tail}
\end{equation}
The recurrence is exact by induction. The initialization is the partial objective of the only one-frame selection ending at $t_i$. If the previous layer is optimal, every $m$-frame selection ending at $t_i$ consists of a feasible $(m-1)$-frame selection ending at some $t_j$, followed by $t_i$. Maximizing over these predecessors therefore gives the optimal partial objective. Equation~\ref{eq:dp-tail} accounts for the only omitted gap, and stored predecessor indices recover an optimal subset by backtracking.

\noindent\textbf{Transforming predecessor searches into line queries.} Assume $\beta>0$. Within a fixed layer, write $d_j=D[m-1,j]$ and define the transition contribution
\begin{equation}
    V_j(x)=d_j+\beta\log\!\left(1+\frac{x-t_j}{\tau}\right),
    \qquad
    L_j(x)=a_j(x+\tau-t_j),\quad a_j=\exp(d_j/\beta).
    \label{eq:envelope-lines}
\end{equation}
For eligible predecessors $j<i$ at $x=t_i$, every $L_j(x)$ is positive and
\begin{equation}
    \max_j V_j(x)=\beta\log\!\left(\frac{\max_j L_j(x)}{\tau}\right).
    \label{eq:envelope-equivalence}
\end{equation}
Thus, the same predecessor maximizes the transition contribution and the corresponding line value.

\noindent\textbf{Discarding dominated predecessors.} Suppose $p<j$ and $d_p\geq d_j$. For every future query $x\geq t_j$,
\begin{equation}
    V_p(x)-V_j(x)
    =d_p-d_j+\beta\log\!\left(\frac{\tau+x-t_p}{\tau+x-t_j}\right)>0.
    \label{eq:record-dominance}
\end{equation}
The earlier predecessor always contributes more, so $j$ can be discarded. We therefore insert only predecessors whose $d_j$ is a strict new maximum among all feasible predecessors processed so far. Their slopes $a_j$ increase strictly, while query timestamps increase with $i$. Importantly, the running maximum includes predecessors already removed from the queue.

\noindent\textbf{Maintaining the upper envelope.} The upper envelope is the pointwise maximum of the retained lines. We maintain it using a monotone upper-envelope construction, storing its potentially optimal lines in a double-ended queue ordered by increasing slope. For retained predecessors $p<q$ with $d_p<d_q$, their intersection is
\begin{equation}
    X(p,q)=t_p-\tau+\frac{t_q-t_p}{1-\exp(-(d_q-d_p)/\beta)}.
    \label{eq:envelope-intersection}
\end{equation}
When inserting a new line $j$, the middle line $q$ of the last three lines $p,q,j$ can be removed whenever $X(p,q)\geq X(q,j)$: it has no interval on which it is strictly best. At a query, if the second line is at least as good as the first, we remove the first. Because the second has larger slope, the first cannot become better at a later query. We insert the newly eligible predecessor $j=i-1$ immediately before querying $t_i$, ensuring that only $j<i$ participates. Algorithm~\ref{alg:fast-exact} gives the full procedure.

\begin{proof}[Proof of Theorem~\ref{thm:dp-optimality}]
For $\beta>0$, Equation~\ref{eq:envelope-equivalence} preserves the maximizing predecessor. Equation~\ref{eq:record-dominance} discards only predecessors dominated at every future query. Back removal preserves the upper envelope, and front removal discards only lines permanently overtaken as query timestamps increase. Consequently, the queue returns a maximizing eligible predecessor for every DP state. The recurrence and terminal boundary reward established above then recover a global maximizer of $F$.

Each predecessor enters the queue at most once and is removed at most once per layer. Insertion and querying therefore require $\mathcal{O}(N)$ total operations per layer. Initialization and terminal maximization take $\mathcal{O}(N)$ time, and backtracking takes $\mathcal{O}(K)$, giving $\mathcal{O}(KN)$ time overall. Two rolling DP rows and the queue use $\mathcal{O}(N)$ working memory; storing backtracking indices uses $\mathcal{O}(KN)$ additional memory. For $\beta=0$, selecting the $K$ largest relevance scores maximizes $F$ directly. Linear-time order-statistic selection followed by a scan of the ordered pool also satisfies the claimed bounds.
\end{proof}

\begin{algorithm}[t]
    \caption{Exact keyframe optimization with a monotone upper envelope}
    \label{alg:fast-exact}
    \small
    \begin{algorithmic}[1]
        \Require Sorted timestamps $t_1<\cdots<t_N$, scores $r_i$, $1\leq K\leq N$, $T>0$, $\beta\geq0$
        \Ensure Optimal indices $J_1<\cdots<J_K$ and objective value $F^\star$
        \If{$\beta=0$}
            \State \Return the $K$ highest-scoring indices in temporal order and their score sum
        \EndIf
        \State $\tau\gets T/(K+1)$; $\mathrm{prev}_i\gets r_i+\beta\log(1+t_i/\tau)$ for $i=1,\ldots,N$
        \State Allocate predecessor array $P[1{:}K,1{:}N]$
        \For{$m=2,\ldots,K$}
            \State $\mathrm{curr}\gets(-\infty,\ldots,-\infty)$; $H\gets$ empty deque; $d_{\max}\gets-\infty$
            \State Define $V_j(x)$ and $X(p,q)$ by Equations~\ref{eq:envelope-lines} and~\ref{eq:envelope-intersection}, with $d_j=\mathrm{prev}_j$
            \For{$i=m,\ldots,N$}
                \State $j\gets i-1$
                \If{$\mathrm{prev}_j>d_{\max}$}
                    \State $d_{\max}\gets\mathrm{prev}_j$
                    \While{$|H|\geq2$ and $X(H_{-2},H_{-1})\geq X(H_{-1},j)$}
                        \State Remove the last element of $H$
                    \EndWhile
                    \State Append $j$ to $H$
                \EndIf
                \While{$|H|\geq2$ and $V_{H_2}(t_i)\geq V_{H_1}(t_i)$}
                    \State Remove the first element of $H$
                \EndWhile
                \State $j^\star\gets H_1$; $\mathrm{curr}_i\gets r_i+V_{j^\star}(t_i)$; $P[m,i]\gets j^\star$
            \EndFor
            \State $\mathrm{prev}\gets\mathrm{curr}$
        \EndFor
        \State $J_K\gets\arg\max_{K\leq i\leq N}[\mathrm{prev}_i+\beta\log(1+(T-t_i)/\tau)]$
        \State $F^\star\gets\mathrm{prev}_{J_K}+\beta\log(1+(T-t_{J_K})/\tau)$
        \For{$m=K,K-1,\ldots,2$}
            \State $J_{m-1}\gets P[m,J_m]$
        \EndFor
        \State \Return $(J_1,\ldots,J_K),F^\star$
    \end{algorithmic}
    \medskip
    \footnotesize $H_1,H_2$ denote the first two queue elements; $H_{-2},H_{-1}$ denote the last two. Equal DP-value records are discarded, and ties at a query favor the later line.
\end{algorithm}

\noindent\textbf{Numerical evaluation and scope.} The lines in Equation~\ref{eq:envelope-lines} establish the equivalence but need not be formed explicitly. We evaluate $V_j(x)$ using $\operatorname{log1p}((x-t_j)/\tau)$ and the denominator in Equation~\ref{eq:envelope-intersection} using $-\operatorname{expm1}(-(d_q-d_p)/\beta)$. This avoids exponentiating large DP values and reduces cancellation for nearby values. The optimality and operation-count statements use real arithmetic, with elementary functions counted as constant-cost operations. Floating-point comparisons near ties can still be sensitive to rounding; avoiding large exponentials is not a guarantee of numerically exact decisions. For $K=1$, initialization and terminal maximization suffice. No pairwise gap-reward matrix is needed, and retaining only the optimal value reduces storage to $\mathcal{O}(N)$.



\FloatBarrier
\section{Details of Adaptive Relevance Scoring}
\label{app:scoring}

Starting from sparse uniform initialization, FORTE predicts relevance for unscored frames and selects additional frames using both predicted relevance and temporal coverage. It then evaluates these frames with the relevance scorer $\psi$ and uses their measured scores to update the GP. Repeating this process allows new observations to inform subsequent scoring decisions until the scoring budget $B$ is exhausted. This section explains how we make GP prediction efficient and organize the select--score--update process into batches.

\subsection{Gaussian-Process Relevance Predictor}
\label{app:gp-approximation}

The GP relevance predictor uses only the timestamps and measured relevance scores of frames already evaluated. Its predictions guide which frames to score next; global keyframe optimization instead uses measured relevance scores. Following Section~\ref{sec:acquisition}, $\mathcal{O}$ denotes the timestamps of the frames scored so far, and $\mathbf r$ contains their normalized relevance scores. The predicted relevance of an unscored frame at timestamp $t$ is $\mu(t)$ in Equation~\ref{eq:gp-mean}, with $\mu_{\mathrm{obs}}$ set to the mean of $\mathbf r$. For frames with no nearby scored frames, predictions approach this observed average rather than zero. Score normalization is described in Appendix~\ref{app:implementation}.

The main computational challenge is updating the GP as the observed set grows. A dense fit accounts for every pair of observed timestamps, even though the RBF kernel assigns small values to temporally distant pairs. We reduce this cost by retaining a band of entries around the kernel matrix's diagonal and evaluating predictions from a local neighborhood. All reported experiments involving GP prediction use the approximation described below.

\noindent\textbf{Fitting the GP with a banded kernel matrix.} Each frame retains its position in the original video's temporal order. At the start of a scoring round, we scan this ordered universe $\mathcal{U}$ and use a Boolean observation mask to collect the $n_{\mathrm{obs}}=|\mathcal{O}|$ observed timestamps and their scores. This produces the required chronological order without comparison sorting. In this order, entries near the kernel matrix's diagonal relate neighboring observations. We keep $[\mathbf K]_{ij}$ when $|i-j|\leq w$ and set the other entries to zero, where $w$ is the half-bandwidth, initialized to $\min\{128,n_{\mathrm{obs}}-1\}$ in our experiments. The band is defined by observation count rather than a fixed temporal duration. The resulting matrix, $\mathbf K^{(w)}$, defines the approximate fit:
\begin{equation}
    (\mathbf K^{(w)}+\epsilon\mathbf{I})\boldsymbol\alpha
    =\mathbf r-\mu_{\mathrm{obs}}\mathbf 1,
    \label{eq:banded-gp-solve}
\end{equation}
where $\mathbf r$ contains the normalized observed scores. The diagonal adjustment $\epsilon\mathbf{I}$ is the same numerical-stability term as in Equation~\ref{eq:gp-mean}. We obtain the coefficients $\boldsymbol\alpha$ using banded Cholesky factorization and triangular solves, without forming a matrix inverse. These coefficients are then reused to predict relevance for unscored frames.

\noindent\textbf{Numerical safeguards.} Zeroing kernel entries outside the band can destroy positive definiteness and cause Cholesky factorization to fail. We therefore allow two forms of recovery: increasing the diagonal adjustment and, if needed, widening the band. All computations used to construct and solve the GP system use double precision (FP64). For a fixed half-bandwidth $w$, we begin with $\epsilon=10^{-12}$ and increase it tenfold after each failed factorization, up to $10^{-3}$. If none of these attempts succeeds and a wider band is allowed, we set $w\leftarrow\min\{2w,w_{\max},n_{\mathrm{obs}}-1\}$, rebuild the matrix, and restart the adjustment schedule from $10^{-12}$. We use a common recovery cap $w_{\max}=1024$ across benchmarks and scoring budgets. We stop at the first successful factorization and use its fitted coefficients and half-bandwidth for prediction. If all attempts fail at the maximum permitted bandwidth, the implementation raises an error rather than returning predictions.

\noindent\textbf{Predicting relevance for unscored frames.} We traverse unscored frames in chronological order while maintaining a pointer into the ordered observed timestamps. As prediction timestamps increase, the pointer advances to identify each frame's insertion position and never moves backward. This locates all prediction neighborhoods in a single linear sweep. For a frame at timestamp $t$, we take at most $2w+1$ adjacent observation indices around its insertion position, denoted $\mathcal{J}_w(t)$. We approximate its predicted relevance $\mu(t)$ by restricting the posterior-mean sum to this neighborhood:
\begin{equation}
    \mu(t)\approx\mu_{\mathrm{obs}}+\sum_{i\in\mathcal{J}_w(t)}k(t,t_i)\alpha_i.
    \label{eq:local-gp-prediction}
\end{equation}
Thus, predicting relevance for another unscored frame requires only a short weighted sum, not another GP fit. This is the local approximation used to compute the relevance predictions in Section~\ref{sec:acquisition}.

\noindent\textbf{Computational complexity.}
We separate the cost of the GP predictor from the remaining computations used to construct each acquisition batch.

\noindent\textit{\textbf{GP preparation, fitting, and prediction.}}
Extracting observations in temporal order takes $\mathcal{O}(|\mathcal{U}|)$ time. Dense GP factorization would require $\mathcal{O}(n_{\mathrm{obs}}^3)$ time and $\mathcal{O}(n_{\mathrm{obs}}^2)$ memory. With half-bandwidth $w$, banded factorization instead requires $\mathcal{O}(n_{\mathrm{obs}}w^2)$ time and $\mathcal{O}(n_{\mathrm{obs}}w)$ memory; constructing the banded matrix and performing the triangular solves each take $\mathcal{O}(n_{\mathrm{obs}}w)$ time. For $n_{\mathrm{pred}}$ prediction frames, the chronological sweep locates all neighborhoods in $\mathcal{O}(n_{\mathrm{obs}}+n_{\mathrm{pred}})$ time, since each observed timestamp is passed at most once. Evaluating the local posterior-mean sums then takes $\mathcal{O}(n_{\mathrm{pred}}w)$ time. The complete GP preparation, fitting, and prediction cost per round is therefore
\begin{equation}
    \mathcal{O}\!\left(
        |\mathcal{U}|
        + n_{\mathrm{obs}}w^2
        + n_{\mathrm{pred}}w
    \right).
    \label{eq:gp-round-complexity}
\end{equation}
Since $n_{\mathrm{obs}},n_{\mathrm{pred}}\leq|\mathcal{U}|$, this is $\mathcal{O}(|\mathcal{U}|)$ for fixed $w$. Including the observation mask and prediction arrays, memory use is
$\mathcal{O}(|\mathcal{U}|+n_{\mathrm{obs}}w)$, also linear for fixed $w$.

Equation~\ref{eq:gp-round-complexity} describes one factorization attempt followed by prediction at the successful bandwidth. Retries add matrix-construction and factorization costs at the attempted bandwidths. With the bounded diagonal-adjustment schedule and fixed recovery cap $w_{\max}$, the retry-inclusive GP cost remains linear in $|\mathcal{U}|$ per round. Across $R$ scoring rounds, the total GP preparation, fitting, and prediction cost is therefore
\begin{equation}
    \mathcal{O}(R|\mathcal{U}|),
    \label{eq:gp-total-complexity}
\end{equation}
or $\mathcal{O}(|\mathcal{U}|)$ for a fixed round count such as the default $R=4$.

\noindent\textit{\textbf{Acquisition and heap maintenance.}}
We next account for the frame-by-frame acquisition decisions after the GP predictions have been computed. Let
$m_0=|\mathcal{O}_0|$ denote the number of uniformly scored initialization frames. Since $m_0\approx B/2$ and these frames are approximately uniformly distributed over the ordered frame universe, they partition $\mathcal{U}$ into temporal gaps containing at most
\begin{equation}
    G_0
    =
    \mathcal{O}\!\left(
        \frac{|\mathcal{U}|}{m_0}
    \right)
    =
    \mathcal{O}\!\left(
        \frac{|\mathcal{U}|}{B}
    \right)
    \label{eq:max-initial-gap}
\end{equation}
unscored candidates per gap. Subsequent acquisitions only split these initial gaps, so every gap encountered during adaptive scoring contains at most $G_0$ candidates.

At the start of each scoring round, the GP predictions $\mu_h(t)$ are fixed for that round. We compute the acquisition value
\[
    \mu_h(t)
    +
    \beta_{\mathrm{acq}}
    \Delta C_{\tau_B}(t\mid \mathcal{A})
\]
for each currently unscored candidate and construct a max-heap over these values. Computing the initial keys by a chronological sweep and heapifying them both require $\mathcal{O}(|\mathcal{U}|)$ time. Across $R$ rounds, this contributes $\mathcal{O}(R|\mathcal{U}|)$ time.

When a frame is selected, only the temporal gap containing that frame is split. Because the marginal coverage gain depends only on the two endpoints of the containing gap, acquisition values outside this gap are unchanged. At most $G_0$ remaining candidates are therefore affected by one selection. Their GP relevance terms remain fixed within the round, while their marginal coverage gains can only decrease as the selected set grows. Using a position-aware max-heap, each affected key can consequently be recomputed and decreased in $\mathcal{O}(\log|\mathcal{U}|)$ time. Thus, one adaptive selection requires
\begin{equation}
    \mathcal{O}\!\left(
        G_0\log|\mathcal{U}|
    \right)
\end{equation}
heap-maintenance time. Since the total number of adaptive selections is at most $B-m_0=\mathcal{O}(B)$, their total cost is
\begin{equation}
    \mathcal{O}\!\left(
        B G_0 \log|\mathcal{U}|
    \right).
    \label{eq:acquisition-heap-general}
\end{equation}
Together with per-round key construction and heapification, the non-GP acquisition cost is therefore
\begin{equation}
    \mathcal{O}\!\left(
        R|\mathcal{U}|
        +
        B G_0\log|\mathcal{U}|
    \right).
    \label{eq:acquisition-misc-complexity}
\end{equation}
The heap, acquisition-value arrays, and gap bookkeeping require $\mathcal{O}(|\mathcal{U}|)$ additional memory.

Under the scoring protocol used in our experiments, the scoring budget satisfies $B=\Theta(\rho T)$, while a video with fixed frame rate $v$ contains $|\mathcal{U}|=\Theta(vT)$ frames. Hence
\begin{equation}
    G_0
    =
    \mathcal{O}\!\left(
        \frac{|\mathcal{U}|}{B}
    \right)
    =
    \mathcal{O}\!\left(
        \frac{v}{\rho}
    \right).
\end{equation}
For fixed frame rate $v$ and scoring allowance $\rho$, $G_0=\mathcal{O}(1)$, and Equation~\ref{eq:acquisition-misc-complexity} simplifies to
\begin{equation}
    \mathcal{O}\!\left(
        R|\mathcal{U}|
        +
        B\log|\mathcal{U}|
    \right).
    \label{eq:acquisition-misc-fixed-density}
\end{equation}
For a fixed number of rounds, this is
\begin{equation}
    \mathcal{O}\!\left(
        |\mathcal{U}|
        +
        B\log|\mathcal{U}|
    \right).
\end{equation}

\noindent\textit{\textbf{Overall adaptive-scoring computation.}}
Combining the GP computation in Equation~\ref{eq:gp-total-complexity} with acquisition and heap maintenance in Equation~\ref{eq:acquisition-misc-complexity}, the complete adaptive-scoring computation, excluding relevance-scorer forward passes, is
\begin{equation}
    \mathcal{O}\!\left(
        R|\mathcal{U}|
        +
        B G_0\log|\mathcal{U}|
    \right)
\end{equation}
for fixed GP bandwidth and recovery cap. Under fixed frame rate, scoring density, and round count, this becomes
\begin{equation}
    \boxed{
    \mathcal{O}\!\left(
        |\mathcal{U}|
        +
        B\log|\mathcal{U}|
    \right)}.
    \label{eq:adaptive-scoring-total-complexity}
\end{equation}
The corresponding working memory is
\[
    \mathcal{O}\!\left(
        |\mathcal{U}|+n_{\mathrm{obs}}w
    \right),
\]
including the GP representation, prediction arrays, max-heap, and acquisition bookkeeping.

These bounds cover observation extraction, GP fitting and posterior-mean prediction, acquisition-value computation, and frame-selection data structures. They do not include video decoding, relevance-scorer inference, final keyframe optimization, or downstream MLLM inference. Appendix~\ref{app:online-cost} reports these pipeline costs separately.

Appendix~\ref{app:gp-validation} reports a replay-based recovery audit, retry-inclusive GP timings, and a separate comparison with dense GP prediction. These checks distinguish successful numerical recovery from agreement with the untruncated predictor.

\subsection{Batch Construction and Observation Updates}

Section~\ref{sec:acquisition} describes choosing and scoring one frame before updating the GP. Although this uses each new score immediately, it requires a GP refit after every evaluation and prevents scoring several selected frames in the same forward pass. Our implementation instead chooses a batch of frames, scores them in parallel, and updates the GP once their scores are available. We call one repetition of these steps a scoring round.

Let $\mathcal{O}_0$ contain the timestamps of the frames scored during uniform initialization. We divide the remaining $B-|\mathcal{O}_0|$ evaluations as evenly as possible among $R$ rounds. Round $h\in\{0,\ldots,R-1\}$ selects $b_h$ new frames, with the allowances summing to the remaining scoring budget. The default uses four rounds; each round proceeds as follows.

\noindent\textbf{Predict from the available scores.} At the start of round $h$, let $\mathcal{O}_h$ contain the timestamps of the frames scored so far and $\mathbf r_h$ their normalized relevance scores in timestamp order. We use the approximation above to compute $\mu_h(t)$, the predicted relevance of an unscored frame at timestamp $t$. The subscript $h$ identifies the current round's observations; the prediction has the same role as $\mu(t)$ in Section~\ref{sec:acquisition}. These predictions remain fixed until the batch has been scored.

\noindent\textbf{Choose frames while updating coverage.} We choose the batch one frame at a time. If coverage were also held fixed, several choices could concentrate around the same predicted relevance peak. We therefore update coverage after every choice, counting both previously scored frames and frames already chosen for the current batch. Let $\mathcal{A}_{h,j}$ contain the timestamps of all these frames after $j$ choices, starting from $\mathcal{A}_{h,0}=\mathcal{O}_h$. The $j$th chosen frame has timestamp $t_{h,j}$, given by
\begin{equation}
    \begin{aligned}
        t_{h,j}&=\arg\max_{t\in\mathcal{U}\setminus\mathcal{A}_{h,j-1}}
        \left[\mu_h(t)+\beta_{\mathrm{acq}}\Delta C_{\tau_B}(t\mid\mathcal{A}_{h,j-1})\right],\\
        \mathcal{A}_{h,j}&=\mathcal{A}_{h,j-1}\cup\{t_{h,j}\},
    \end{aligned}
    \label{eq:acquisition-batch}
\end{equation}
for $j=1,\ldots,b_h$, with coverage evaluated over $[0,T]$ and $\tau_B=T/(B+1)$. The relevance term uses the predictions made at the start of the round, while the coverage term reflects all choices made so far. Updating coverage requires only the timestamps of the chosen frames, not their measured relevance scores.

This frame-by-frame update is inexpensive: each choice splits one temporal gap, changing coverage gains only for candidates inside it. By Equation~\ref{eq:marginal-coverage}, evaluating a candidate's gain uses only its neighboring selected timestamps and takes $\mathcal{O}(1)$ arithmetic operations. Maintaining the ordered timestamps and finding the next best candidate are separate from this local gain calculation.

\noindent\textbf{Score the batch and update the observations.} After all $b_h$ frames have been chosen, we evaluate them with the relevance scorer using batched forward passes. These evaluations are independent, so all selected frames can be scored in parallel given sufficient computational resources. Once all scores are available, we add them to the observations and set $\mathcal{O}_{h+1}=\mathcal{A}_{h,b_h}$. If budget remains, these observations determine the GP predictions for the next round.

Within each round, GP predictions remain fixed, while temporal coverage is updated after each frame choice. Once the selected frames have been scored, their measured relevance scores are used to update the GP for the next round.

\section{Implementation and Evaluation Details}
\label{app:implementation}

\subsection{FORTE Implementation}

\noindent\textbf{Default settings.} Unless otherwise stated, we use the RBF kernel in Equation~\ref{eq:rbf-kernel} with length scale $\ell=10$ seconds. The coverage weights are $\beta_{\mathrm{acq}}=3.5$ for adaptive relevance scoring and $\beta_{\mathrm{sel}}=1$ for global keyframe optimization. We allocate approximately half of the scoring budget to uniform initialization and divide the remainder among four scoring rounds. We use the $\mathcal{O}(KN)$ exact algorithm for final frame selection. Both stages of FORTE are deterministic given fixed scorer outputs and hyperparameters.

\noindent\textbf{Hyperparameter choices and numerical safeguards.} We use heuristic defaults shared across benchmarks and scoring budgets, without development-set tuning. With half-budget uniform initialization, a typical interior gap is approximately $2\tau_B$, and scoring its midpoint yields a coverage gain of $\log(4/3)$. We set $\beta_{\mathrm{acq}}\approx1/\log(4/3)\approx3.5$ so that this reference exploration bonus is comparable to the unit relevance range. For final optimization, $\beta_{\mathrm{sel}}=1$ provides a simple unit-weight coverage regularizer alongside measured relevance. The GP length scale of $10$ seconds, initialization fraction of $0.5$, and four scoring rounds are simple fixed defaults. The sensitivity experiments in Appendix~\ref{app:sensitivity} support these choices and indicate that performance is robust to moderate parameter changes. We also use a common numerical-recovery rule: FP64 arithmetic, an initial half-bandwidth of $\min\{128,n_{\mathrm{obs}}-1\}$, the same diagonal-adjustment schedule, and a recovery cap of $w_{\max}=1024$ (Appendix~\ref{app:gp-approximation}). This cap limits expansion after failed Cholesky attempts; it does not prescribe the bandwidth used in every fit.

\noindent\textbf{Setting the scoring budget.} The scoring allowance $\rho$ determines how many frame--query evaluations are available per second of video, not which frames are scored. We obtain the integer budget $B$ by counting distinct frames on a uniform reference grid. For a video with $n>1$ original frames at frame rate $v$, the timestamp span is $T=(n-1)/v$. We compute the reference stride $d_\rho=\max(1,\lfloor v/\rho\rfloor)$ and construct $\max(1,\lfloor n/d_\rho\rfloor)$ reference points. Mapping these points to original frames and removing duplicates gives $B$.

Baselines that score uniformly sampled candidates use the frames on this reference grid. FORTE instead scores an initial grid at $\rho/2$ and allocates the remaining evaluations adaptively. Thus, budget-matched methods evaluate the same number of frames but need not evaluate the same frames. Integer rounding can make the initialization count $|\mathcal{O}_0|$ differ slightly from $B/2$. We divide the remaining $B-|\mathcal{O}_0|$ evaluations as evenly as possible among the scoring rounds, so their batch sizes differ by at most one frame.

\noindent\textbf{Normalizing relevance scores.} We min--max normalize the observed relevance scores to $[0,1]$ for each question, placing relevance on a common scale before combining it with coverage. If all observed scores are identical, their normalized values are set to zero. GP predictions are clipped to the range of the normalized observations. The final scores $r_i$ remain fixed throughout global keyframe optimization.

\noindent\textbf{Implementing the exact algorithm.} We maintain two rolling DP rows and a double-ended queue for the current layer, using $\mathcal{O}(N)$ working memory, and store $\mathcal{O}(KN)$ predecessor indices for backtracking. Transition values and line intersections are evaluated as described in Appendix~\ref{app:proof-dp-optimality}; no pairwise gap-reward matrix is materialized. Backtracking returns the selected frames in chronological order.

\subsection{Evaluation Protocol}

\noindent\textbf{Dataset splits and question counts.} We evaluate the LongVideoBench validation split ($1{,}337$ questions), the Video-MME test split ($2{,}700$ questions), the public LVBench evaluation set ($1{,}549$ questions), and the multiple-choice portion of the MLVU development set ($2{,}174$ questions). For reproducibility, the LVBench annotations come from the split named \texttt{train} in \texttt{lmms-lab/LVBench}; this is the public benchmark annotation set, not data used to train our method. The MLVU annotations come from the \texttt{mlvu} configuration and \texttt{test} split of \texttt{sy1998/MLVU\_dev}.

\noindent\textbf{Available input modalities.} We use video frames without audio or separately supplied subtitles on all four benchmarks. In particular, we do not use LongVideoBench's interleaved-subtitle input or Video-MME's subtitle-augmented input. Neither the scorer nor the downstream MLLM receives speech transcripts, generated captions, or separately extracted OCR text. Text embedded in the video remains visible in the selected frames. The downstream MLLM receives the question and its answer options, whereas relevance scoring uses the question alone.

\noindent\textbf{Scorer checkpoints and scoring heads.} Our default scorer is BLIP ViT-L, loaded through LAVIS as \texttt{blip\_image\_text\_matching} with model type \texttt{large}. This configuration uses the COCO-retrieval checkpoint \nolinkurl{model_large_retrieval_coco.pth} and $384\times384$ image preprocessing.\footnote{Official checkpoint configuration: \url{https://github.com/salesforce/LAVIS/blob/main/lavis/configs/models/blip_itm_large.yaml}.} We use the image--text matching (ITM) head and take the positive-class probability after the two-class softmax as the relevance score, rather than using the contrastive head.

For BLIP-2, we use \texttt{Salesforce/blip2-itm-vit-g} with its image--text retrieval model and associated processor. Its ITM head produces two-class logits for the Q-Former query tokens; these logits are averaged across the query tokens before applying softmax, and the positive-class probability is the frame's relevance score. For CLIP, we use \texttt{openai/clip-vit-large-patch14} and compute cosine similarity between the projected, $\ell_2$-normalized image and text embeddings, without logit-temperature scaling. Each scorer uses its corresponding evaluation image processor and tokenizer. All three receive only the original question text: we do not concatenate answer options or rewrite the question for scoring. The resulting scalar scores are normalized as described above before they enter FORTE.

\noindent\textbf{Downstream prompting and generation.} We use each downstream model's chat or conversation template with the selected frames and a multiple-choice question. The text contains the question followed by one answer option per line, using \texttt{A.}, \texttt{B.}, etc. for LongVideoBench, the supplied option labels for Video-MME, and \texttt{(A)}, \texttt{(B)}, etc. for LVBench and MLVU. It ends with ``Answer with the option's letter from the given choices directly.'' The Video-MME task additionally prepends ``Select the best answer to the following multiple-choice question based on the video. Respond with only the letter (A, B, C, or D) of the correct option.'' This inherited instruction mentions subtitles, but no subtitle text is supplied. We provide no in-context examples or chain-of-thought instruction. Answer generation uses greedy decoding without stochastic sampling, with a maximum of $16$ new tokens for Video-MME and $32$ for the other benchmarks.

\noindent\textbf{Answer extraction.} We compare the extracted option letter with the reference answer, without an LLM judge. The LongVideoBench and Video-MME task parsers strip standard answer prefixes and extract the first uppercase option letter (\texttt{A}--\texttt{E} and \texttt{A}--\texttt{D}, respectively). LVBench and MLVU use the evaluation harness's multiple-choice label extractor, which recognizes option letters in standard answer formats. An unparseable answer is counted as incorrect.

\noindent\textbf{Original timestamps and nonuniform frame inputs.} All selected frames are supplied in chronological order, retaining their original video-frame indices. For original frame index $i$ and source video frame rate $v$, its timestamp is $i/v$ seconds. In Qwen3-VL, the selected frames are passed as a video tensor together with the original frame indices, source frame rate, and total number of original frames. The processor uses this metadata to represent their temporal positions; further frame sampling is disabled, so a nonuniform selection is not replaced by a uniformly sampled sequence. In the downstream-model comparisons, InternVL3.5 receives a textual prefix specifying the video duration and the selected frames' original timestamps in seconds. LLaVA-Video and LLaVA-OneVision receive the chronologically ordered frames without an additional timestamp instruction. These input conventions are held fixed across selectors for each downstream model.

\noindent\textbf{Accuracy aggregation.} We report the percentage of correctly answered questions for each benchmark, weighting all evaluated questions equally. The four-benchmark mean, \emph{Avg.}, is the unweighted mean of these accuracies.

\noindent\textbf{Completing short selections.}\label{app:completion-policy} We adopt a common \emph{preserve-and-supplement} rule for all methods. If a selector returns fewer than $K$ frames, we retain its selected frames and repeatedly add the original video frame farthest in time from the current set until $K$ frames are obtained, breaking ties in favor of earlier timestamps. The added frames require no relevance evaluations and therefore do not consume additional scoring budget.

\subsection{Baseline Implementations}
\label{app:baseline-implementations}

We compare baseline frame-selection rules within a common evaluation pipeline, using Qwen3-VL-8B-Instruct and BLIP unless otherwise stated. In the scorer ablation, we change the relevance scorer while keeping each selection rule and its parameters fixed. In the downstream-model ablation, we evaluate the same selected frames with different MLLMs. All selected frames are supplied in chronological order. The descriptions below identify the official implementations and specify our adaptations to this common protocol.

All score-based baselines except FOCUS score a uniformly sampled candidate pool before selecting the final input frames. We refer to these as \emph{fixed-grid baselines} and match their scoring allowance $\rho$ and final-frame budget $K$ to FORTE. FOCUS follows its official scoring procedure, detailed below. Appendix~\ref{app:completion-policy} specifies the common rule for completing short selections.

\noindent\textbf{Uniform and Top-K.} \textsc{Uniform} samples $K$ evenly spaced frames from the original video without relevance scoring. \textsc{Top-K} selects the highest-scoring frames from the observed candidate pool.

\noindent\textbf{AKS.}\footnote{\url{https://github.com/ncTimTang/AKS}.}~\citep{aks} We follow the official AKS implementation and its default normalization and threshold settings, setting the maximum recursion depth to $\log_2 K$ for the tested frame budgets.

\noindent\textbf{BOLT.}\footnote{\url{https://github.com/sming256/BOLT}; \href{https://github.com/sming256/BOLT/blob/main/LICENSE}{MIT license}.}~\citep{bolt} We follow BOLT's official inverse-transform sampler and default hyperparameters.

\noindent\textbf{Q-Frame.}\footnote{\url{https://github.com/xiaomi-research/q-frame}; \href{https://github.com/xiaomi-research/q-frame/blob/master/LICENSE}{Apache-2.0 license}.}~\citep{q-frame} We follow Q-Frame's official implementation and default selection temperature. We keep downstream image preprocessing consistent across methods.



\noindent\textbf{AdaQ.}\footnote{\url{https://github.com/Zkayovo-xmu/AdaQ}.}~\citep{adaq} We follow AdaQ's official implementation, retaining its default variance-scaling coefficient, top-$p$ threshold, and minimum-support parameter. 


\noindent\textbf{WFS-SB.}\footnote{\url{https://github.com/MAC-AutoML/WFS-SB}.}~\citep{wfssb} We follow the official WFS-SB implementation and default selection hyperparameters, using BLIP relevance scores and visual features in our common evaluation pipeline.


\noindent\textbf{KFC from Nar-KFC.}\footnote{\url{https://github.com/bofang98/Nar-KFC}; \href{https://github.com/bofang98/Nar-KFC/blob/main/LICENSE}{Apache-2.0 license}.}~\citep{narkfc} We follow the official greedy KFC implementation, retaining its diversity and low-rank settings. We fix the candidate-reduction target and the local-refinement radius across benchmarks.

\noindent\textbf{FOCUS.}\footnote{\url{https://github.com/NUS-HPC-AI-Lab/FOCUS}; \href{https://github.com/NUS-HPC-AI-Lab/FOCUS/blob/main/LICENSE}{Apache-2.0 license}.}~\citep{focus} We follow the official FOCUS implementation and default selection hyperparameters, using BLIP as the scorer. Its scoring schedule determines the number of evaluations rather than enforcing a prescribed $\rho$, yielding an effective allowance of approximately $0.44$ frame--query evaluations per second of video in the full-benchmark comparison.

\noindent\textbf{Randomness in stochastic selectors.} Q-Frame uses Gumbel-Max sampling, and AdaQ also samples frames from a relevance-dependent distribution. To account for this randomness, we evaluate each method with five different random seeds, keeping the candidate pools, relevance scores, hyperparameters, and downstream inference settings fixed within each experimental condition. For each seed, we evaluate the selected frames and compute accuracy separately on each benchmark; the reported accuracy in Table~\ref{tab:main-results} is the arithmetic mean across the five runs. Avg.\ is then the unweighted mean of the four seed-averaged benchmark accuracies. Downstream answer generation uses greedy decoding, so the seeds vary frame selection rather than token sampling.

\section{Additional Experimental Results}
\label{app:additional-experiments}

Unless otherwise stated, experiments use Qwen3-VL-8B-Instruct, BLIP, $\rho=0.5$, $K=64$, and the exact algorithm for global keyframe optimization, following Appendix~\ref{app:implementation}. Accuracy is question-weighted within each benchmark; Avg.\ weights the four benchmarks equally.

\subsection{Matched-Budget Comparison with FOCUS}
\label{app:focus-matched-budget}

Table~\ref{tab:main-results} places FOCUS, whose official procedure yields an effective scoring allowance of approximately $0.44$, alongside methods evaluated at $\rho=0.5$. To separate the effect of frame selection from this budget difference, Table~\ref{tab:focus-matched-budget} compares FOCUS and FORTE with identical scoring counts for each question. FOCUS retains its official procedure, and FORTE receives the same number of relevance evaluations as the number of distinct frames scored by FOCUS for each video--question pair. Only the count is shared; each method chooses its own frames. This matches the integer budgets directly, including cases where FOCUS's per-video scoring rate differs from its aggregate rate of approximately $0.44$.

Both methods use BLIP, Qwen3-VL-8B-Instruct, $K=64$, and the same question sets, prompting, decoding, and preserve-and-supplement rule. FORTE retains its default hyperparameters, initialization fraction, four scoring rounds, and exact final optimizer; only its total scoring budget changes.

\begin{table}[htbp]
    \centering
    \caption{Matched-budget comparison at an effective scoring allowance of approximately $\rho=0.44$, using BLIP, Qwen3-VL-8B-Instruct, and $K=64$. Both methods score the same number of distinct frames for each question. Values are QA accuracy (\%); Avg.\ is the unweighted four-benchmark mean.}
    \label{tab:focus-matched-budget}
    \small
    \setlength{\tabcolsep}{3pt}
    \begin{tabular*}{\linewidth}{@{\extracolsep{\fill}}lcrrrrr@{}}
        \toprule
        Method & Effective $\rho$ & LongVB & Video-MME & LVBench & MLVU & Avg. \\
        \midrule
        FOCUS & $\sim\!0.44$ & 64.1 & 65.3 & 48.7 & 71.3 & 62.4 \\
        FORTE (Ours) & $\sim\!0.44$ & 67.8 & 69.7 & 50.5 & 75.2 & 65.8 \\
        \bottomrule
    \end{tabular*}
\end{table}

\noindent\textbf{Results.} FORTE achieves higher observed accuracy than FOCUS on all four benchmarks under matched scoring budgets, with a mean of $65.8\%$ compared with $62.4\%$. The accuracy advantage observed in Table~\ref{tab:main-results} therefore persists when both methods receive the same number of relevance evaluations, rather than arising solely from FORTE's larger scoring allowance in that comparison.

\subsection{Contributions of the Two Stages}
\label{app:component-details}

Table~\ref{tab:component-details} breaks down the main-text ablation in Table~\ref{tab:component-ablation} by benchmark.

\noindent\textbf{Stage 1: Adaptive relevance scoring.}
With Stage 2 fixed to global keyframe optimization, combining GP-predicted relevance with temporal coverage improves mean accuracy from $65.2\%$ with uniform scoring to $66.1\%$, and also outperforms either single-term variant. Coverage-only scoring matches uniform scoring at $65.2\%$, whereas relevance-only scoring achieves $64.2\%$, indicating that the benefit comes from jointly exploiting predicted relevance and coverage rather than either term alone. When Stage 2 is instead fixed to Top-K, adaptive scoring yields a smaller improvement over uniform scoring, from $64.6\%$ to $64.9\%$. Thus, adaptive scoring improves the quality of the candidate pool, but its downstream benefit depends on how the final frames are selected; a relevance-only Top-K selector cannot fully exploit candidates that are useful for improving temporal coverage.

\noindent\textbf{Stage 2: Global keyframe optimization.}
On identical scored pools, replacing Top-K with global keyframe optimization improves mean accuracy for both scoring policies: from $64.6\%$ to $65.2\%$ on the uniformly scored pool and from $64.9\%$ to $66.1\%$ on the adaptively scored pool. The larger gain on the adaptive pool indicates that the two stages are complementary: adaptive scoring makes more informative and temporally distributed candidates available, while global optimization determines which of these candidates should be retained by jointly considering measured relevance and coverage.

\begin{table}[!ht]
    \centering
    \caption{Contributions of the scoring policy and final selector under matched scoring budgets and completion rules. ``Global opt.'' denotes global keyframe optimization of the joint relevance--coverage objective. Values are accuracy (\%); bold marks the best result in each column.}
    \label{tab:component-details}
    \small
    \setlength{\tabcolsep}{2pt}
    \newcommand{\appendixhead}[1]{\makebox[0pt][c]{#1}}
    \begin{tabular*}{\linewidth}{@{\extracolsep{\fill}}llccccc@{\hspace{2pt}}}
        \toprule
        Stage 1 & Stage 2 & \appendixhead{LongVB} & \appendixhead{Video-MME} & \appendixhead{LVBench} & \appendixhead{MLVU} & \appendixhead{Avg.} \\
        \midrule
        Uniform & Top-K & 65.7 & 68.1 & 50.6 & 73.8 & 64.6 \\
        Uniform & Global opt. & 66.6 & 69.2 & 50.4 & 74.5 & 65.2 \\
        Coverage only & Global opt. & 65.9 & 69.6 & 50.4 & 74.9 & 65.2 \\
        Relevance only & Global opt. & 66.4 & 67.9 & 49.6 & 73.0 & 64.2 \\
        Relevance + coverage & Top-K & 66.7 & 67.5 & \textbf{51.0} & 74.4 & 64.9 \\
        Relevance + coverage & Global opt. & 68.1 & \textbf{69.9} & 50.7 & \textbf{75.7} & \textbf{66.1} \\
        \bottomrule
    \end{tabular*}
\end{table}

\subsection{Effect of the Final Optimization Algorithm}
\label{app:solver-comparison}

Unlike the Top-K comparison, Table~\ref{tab:solver-comparison} isolates the optimizer by holding the objective in Equation~\ref{eq:final-objective} fixed. Exact and marginal-gain greedy optimization receive identical pools and measured scores under both uniform and adaptive scoring, with $\beta_{\mathrm{sel}}=1$ and the same completion rule.

The greedy algorithm, labeled ``Approx.'', starts from the empty set and repeatedly adds the frame with the largest gain $r_i+\beta_{\mathrm{sel}}\Delta C_{\tau_K}(t_i\mid\mathcal{S})$ until $K$ frames have been selected, keeping $\tau_K=T/(K+1)$ fixed.

\begin{table}[!ht]
    \centering
    \caption{Exact versus approximate optimization of the same relevance--coverage objective at $\rho=0.5$ and $K=64$. Approx. denotes marginal-gain greedy maximization. Accuracy is reported in percent; Avg. is the unweighted four-benchmark mean. $\Delta F=F(\mathcal{S}_{\mathrm{exact}})-F(\mathcal{S}_{\mathrm{solver}})$ measures the objective shortfall relative to the exact optimum, in objective units. Objective gaps and final-selection times are evaluated on pools with $N\geq K$, averaged within each benchmark, then equally across benchmarks. Timing excludes relevance scoring and downstream inference.}
    \label{tab:solver-comparison}
    \small
    \setlength{\tabcolsep}{2pt}
    \begin{tabular*}{\linewidth}{@{\extracolsep{\fill}}llrrrrrrr@{}}
        \toprule
        \multirow{2}{*}{Scoring} & \multirow{2}{*}{Solver} & \multicolumn{5}{c}{QA accuracy (\%)} & \multirow{2}{*}{\shortstack{Mean\\$\Delta F$}} & \multirow{2}{*}{\shortstack{Time\\(ms)}} \\
        \cmidrule(lr){3-7}
        & & LongVB & Video-MME & LVBench & MLVU & Avg. & & \\
        \midrule
        Uniform & Exact & 66.6 & 69.2 & 50.4 & 74.5 & 65.2 & 0.0 & 45 \\
        Uniform & Approx. & 66.1 & 68.9 & 50.5 & 74.1 & 64.9 & 0.4 & 42 \\
        \midrule
        Adaptive & Exact & 68.1 & 69.9 & 50.7 & 75.7 & 66.1 & 0.0 & 50 \\
        Adaptive & Approx. & 67.3 & 69.4 & 50.8 & 75.3 & 65.7 & 0.7 & 47 \\
        \bottomrule
    \end{tabular*}
\end{table}

\noindent\textbf{Results.} Exact optimization improves mean accuracy by $0.3$ points on uniform pools and $0.4$ points on adaptive pools, eliminating mean objective gaps of $0.4$ and $0.7$, respectively, at an additional $3$\,ms per selection. Accuracy improves on LongVideoBench, Video-MME, and MLVU under both policies; Approx. is slightly higher on LVBench. The objective guarantee therefore accompanies a mean QA improvement, not a guarantee of higher accuracy on every benchmark.

\subsection{Paired Uncertainty Analysis}
\label{app:paired-uncertainty}

FORTE is deterministic given fixed scorer outputs and hyperparameters, and answer generation uses greedy decoding, so we do not repeat FORTE inference with different random seeds. Here, we quantify video-sampling uncertainty through a paired video-cluster bootstrap. This is distinct from averaging Q-Frame and AdaQ over five selector seeds (Appendix~\ref{app:baseline-implementations}); those two methods are not included in the paired comparisons below.

We compare default FORTE with Uniform, AKS, FOCUS, and WFS-SB. All share the default model and frame budget; FOCUS retains its effective scoring allowance of approximately $0.44$, and Uniform requires no relevance scoring.

\noindent\textbf{Measuring paired differences.} Let $c^{\mathrm{F}}_{di},c^{\mathrm{b}}_{di}\in\{0,1\}$ indicate whether FORTE and a baseline correctly answer question $i$ in benchmark $d$. We compute their accuracy difference in percentage points as
\begin{equation}
    \widehat{\Delta}_d=\frac{100}{n_d}\sum_{i=1}^{n_d}\left(c^{\mathrm{F}}_{di}-c^{\mathrm{b}}_{di}\right),
    \qquad
    \widehat{\Delta}_{\mathrm{avg}}=\frac{1}{4}\sum_{d=1}^{4}\widehat{\Delta}_d,
    \label{eq:paired-accuracy-difference}
\end{equation}
where $n_d$ is the number of evaluated questions in benchmark $d$. Pairing keeps each question's outcomes together rather than treating the two methods' accuracies as independent measurements.

\noindent\textbf{Video-cluster bootstrap.} Questions about the same video may have correlated outcomes, so we resample whole videos rather than individual questions. Within each benchmark, a replicate samples the original number of videos with replacement and retains all questions from each sampled video, including repeated occurrences. The same sampled videos are used for all five methods. We recompute question-weighted accuracy differences within each benchmark, then average the four differences with equal weight. Resampling is independent across benchmarks.

We use $20{,}000$ replicates with random seed $42$. The 2.5th and 97.5th percentiles define each 95\% confidence interval. These intervals describe video-sampling uncertainty conditional on the evaluated models and settings, not variation across inference runs.

\noindent\textbf{Results.} At the default setting, the 95\% confidence intervals for the four-benchmark mean differences against Uniform, AKS, FOCUS, and WFS-SB all lie above zero (Table~\ref{tab:paired-uncertainty}), supporting a mean accuracy advantage over each of these baselines. FORTE also has higher observed accuracy in every benchmark-level comparison, with intervals excluding zero against Uniform on all four benchmarks.

\begin{table}[htbp]
    \centering
    \caption{Paired comparisons of FORTE with four baselines. Each baseline has two columns: the accuracy difference $\Delta$ (FORTE minus baseline, in percentage points) and its 95\% video-cluster bootstrap confidence interval. Positive differences favor FORTE. Avg.\ is the unweighted four-benchmark mean. Its 95\% confidence interval is obtained by averaging the four benchmark differences within each bootstrap replicate, then taking the 2.5th and 97.5th percentiles of these averages.}
    \label{tab:paired-uncertainty}
    \small
    \setlength{\tabcolsep}{2pt}
    \renewcommand{\arraystretch}{1.1}
    \begin{tabular*}{\linewidth}{@{\extracolsep{\fill}}l*{4}{cc}@{}}
        \toprule
        \multirow{2}{*}{Benchmark} & \multicolumn{2}{c}{Uniform} & \multicolumn{2}{c}{AKS} & \multicolumn{2}{c}{FOCUS} & \multicolumn{2}{c}{WFS-SB} \\
        \cmidrule(lr){2-3}\cmidrule(lr){4-5}\cmidrule(lr){6-7}\cmidrule(lr){8-9}
        & $\Delta$ & 95\% CI & $\Delta$ & 95\% CI & $\Delta$ & 95\% CI & $\Delta$ & 95\% CI \\
        \midrule
        LongVideoBench & $+5.7$ & $[3.7,\,7.7]$ & $+3.0$ & $[1.2,\,4.8]$ & $+4.0$ & $[1.9,\,6.0]$ & $+2.6$ & $[0.7,\,4.5]$ \\
        Video-MME & $+2.2$ & $[0.9,\,3.6]$ & $+0.9$ & $[-0.3,\,2.2]$ & $+4.6$ & $[3.3,\,6.0]$ & $+1.5$ & $[0.3,\,2.8]$ \\
        LVBench & $+9.7$ & $[7.0,\,12.3]$ & $+3.0$ & $[0.8,\,5.3]$ & $+2.0$ & $[-0.4,\,4.4]$ & $+0.1$ & $[-1.8,\,2.0]$ \\
        MLVU & $+5.4$ & $[3.7,\,7.2]$ & $+1.2$ & $[-0.3,\,2.7]$ & $+4.4$ & $[2.8,\,5.9]$ & $+1.0$ & $[-0.4,\,2.4]$ \\
        \midrule
        Avg. & $+5.8$ & $[4.8,\,6.8]$ & $+2.0$ & $[1.1,\,2.9]$ & $+3.8$ & $[2.9,\,4.8]$ & $+1.3$ & $[0.5,\,2.2]$ \\
        \bottomrule
    \end{tabular*}
\end{table}

\FloatBarrier

\subsection{Parameter Sensitivity}
\label{app:sensitivity}

Figure~\ref{fig:parameter-sensitivity} varies one parameter at a time around the heuristic defaults in Appendix~\ref{app:implementation}, reporting four-benchmark mean accuracy. These sweeps assess sensitivity rather than select defaults. The initialization and round-count sweeps keep the total scoring budget fixed.

\begin{figure}[!ht]
    \centering
    \includegraphics[width=\linewidth]{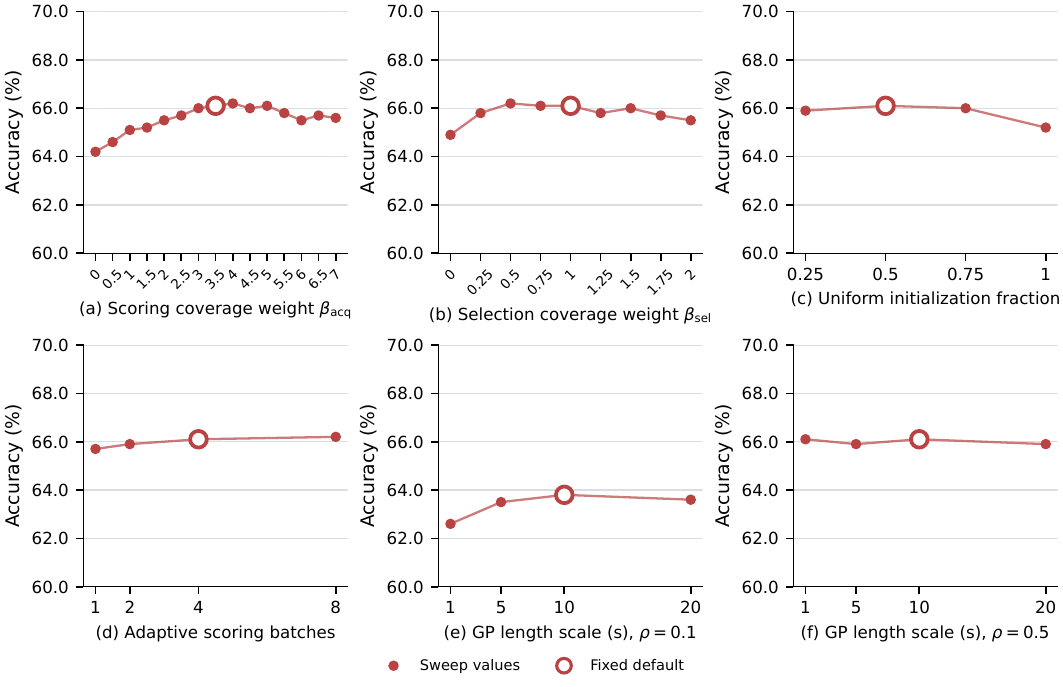}
    \caption{Parameter sensitivity measured by the unweighted four-benchmark mean accuracy (\%). We vary one parameter at a time around the heuristic defaults, marked by hollow circles. Panels (a)--(d) use $\rho=0.5$; panels (e)--(f) use the indicated budgets. All panels use $K=64$.}
    \label{fig:parameter-sensitivity}
\end{figure}

\noindent\textbf{Scoring coverage weight.} In Figure~\ref{fig:parameter-sensitivity}(a), every positive weight improves on relevance-only scoring. Accuracy rises up to $\beta_{\mathrm{acq}}=3$, remains near its peak from $3$ to $5$, and declines slightly at larger weights. The default of $3.5$ lies within this broad high-performing region.

\noindent\textbf{Selection coverage weight.} In Figure~\ref{fig:parameter-sensitivity}(b), all positive weights outperform relevance-only optimization. Accuracy peaks at $\beta_{\mathrm{sel}}=0.5$ but remains similar from $0.25$ to $1.5$, including the default of $1$, before declining at the two largest weights. We retain the unit-weight default rather than the sweep's best-performing value, so the main results reflect a shared heuristic setting rather than a weight selected to maximize evaluation accuracy.

\noindent\textbf{Initialization fraction.} Accuracy varies by only $0.2$ points across fractions from $0.25$ to $0.75$, with $0.5$ giving the highest result (Figure~\ref{fig:parameter-sensitivity}(c)). All three outperform spending the entire budget on uniform scoring: reserving evaluations for adaptive scoring matters more than the precise split within this range.

\noindent\textbf{Number of scoring rounds.} More frequent GP updates improve accuracy at a fixed scoring budget (Figure~\ref{fig:parameter-sensitivity}(d)), but doubling the rounds from four to eight adds only $0.1$ points. Four rounds retain most of the observed benefit with fewer refits.

\noindent\textbf{GP length scale.} At $\rho=0.1$, accuracy increases from $\ell=1$ to $10$ seconds and decreases slightly at $20$ seconds (Figure~\ref{fig:parameter-sensitivity}(e)--(f)). At $\rho=0.5$, all scales lie within $0.2$ points, with $1$ and $10$ seconds tied for the highest accuracy. A broader neighborhood thus appears more useful under sparse observations; the shared $10$-second default is highest or joint-highest at both budgets.

\subsection{Numerical Validation of the GP Predictor}
\label{app:gp-validation}

We assess numerical recovery, its cost, and agreement with dense GP prediction separately: successful factorization does not by itself establish approximation accuracy.

\noindent\textbf{Numerical recovery.} We replay adaptive scoring on all four benchmarks at $\rho\in\{0.1,0.5,1.0\}$ using the original numerical settings, revealing saved scores only when the policy selects their frames. All $23{,}280$ question--budget pairs exactly reproduce their saved scored pools.

All $91{,}247$ fits succeed (Table~\ref{tab:gp-stability}); $1.95\%$ require an increased diagonal adjustment and $0.55\%$ a wider band. No retries occur at $\rho=0.1$, and they are more frequent at the largest budget. The successful half-bandwidth is at most $128$ for $99.45\%$ of fits; the diagonal adjustment remains $10^{-12}$ for $98.48\%$ and never exceeds $10^{-3}$.

\begin{table}[!ht]
    \centering
    \caption{GP recovery in the scoring replay. Counts and percentages indicate fits requiring any increase in the diagonal adjustment or initial half-bandwidth. All fits succeed without terminal errors. The common recovery cap is $w_{\max}=1024$; the largest successful half-bandwidth $w$ is reported separately.}
    \label{tab:gp-stability}
    \small
    \setlength{\tabcolsep}{2pt}
    \begin{tabular*}{\linewidth}{@{\extracolsep{\fill}}lcrrrrrrr@{}}
        \toprule
        \multirow{2}{*}{Benchmark} & \multirow{2}{*}{$\rho$} & \multirow{2}{*}{Fits} & \multicolumn{2}{c}{Adjustment increase} & \multicolumn{2}{c}{Bandwidth expansion} & \multirow{2}{*}{$w_{\max}$} & \multirow{2}{*}{\shortstack{Largest\\successful $w$}} \\
        \cmidrule(lr){4-5}\cmidrule(lr){6-7}
        & & & Count & Rate (\%) & Count & Rate (\%) & & \\
        \midrule
        LongVideoBench & 0.1 & 4,195 & 0 & 0.00 & 0 & 0.00 & 1024 & 128 \\
        Video-MME & 0.1 & 10,302 & 0 & 0.00 & 0 & 0.00 & 1024 & 128 \\
        LVBench & 0.1 & 6,196 & 0 & 0.00 & 0 & 0.00 & 1024 & 128 \\
        MLVU & 0.1 & 8,696 & 0 & 0.00 & 0 & 0.00 & 1024 & 128 \\
        \midrule
        LongVideoBench & 0.5 & 5,132 & 38 & 0.74 & 9 & 0.18 & 1024 & 512 \\
        Video-MME & 0.5 & 10,794 & 19 & 0.18 & 4 & 0.04 & 1024 & 256 \\
        LVBench & 0.5 & 6,196 & 49 & 0.79 & 13 & 0.21 & 1024 & 512 \\
        MLVU & 0.5 & 8,696 & 17 & 0.20 & 7 & 0.08 & 1024 & 512 \\
        \midrule
        LongVideoBench & 1.0 & 5,348 & 336 & 6.28 & 117 & 2.19 & 1024 & 1024 \\
        Video-MME & 1.0 & 10,800 & 284 & 2.63 & 55 & 0.51 & 1024 & 512 \\
        LVBench & 1.0 & 6,196 & 629 & 10.15 & 151 & 2.44 & 1024 & 512 \\
        MLVU & 1.0 & 8,696 & 406 & 4.67 & 145 & 1.67 & 1024 & 1024 \\
        \bottomrule
    \end{tabular*}
\end{table}

\noindent\textbf{Cost of numerical recovery.} We measure GP fitting and prediction time on a duration-stratified subset of $100$ videos, with $25$ per benchmark and one question per video. Across the three budgets and all scoring rounds, this gives $1{,}148$ observed sets on which to fit the GP. We time each set three times on the same CPU host with one math-library thread, then average the repeats before computing the summary statistics.

Table~\ref{tab:gp-retry-cost} includes every factorization attempt and its matrix construction. Failed attempts account for $0.097\%$ of GP time; prediction for unscored frames dominates. End-to-end costs are reported in Appendix~\ref{app:online-cost}.

\begin{table}[!ht]
    \centering
    \caption{GP time per observed set on the 100-video subset. Fitting includes all retries. The last row reports the portion spent on unsuccessful attempts, already included above; its 95th percentile is zero because fewer than $5\%$ of the observed sets require retries.}
    \label{tab:gp-retry-cost}
    \small
    \setlength{\tabcolsep}{2pt}
    \begin{tabular*}{\linewidth}{@{\extracolsep{\fill}}lrr@{}}
        \toprule
        Component & Mean (ms) & 95th percentile (ms) \\
        \midrule
        Matrix construction & 1.229 & 2.708 \\
        Factorization & 0.653 & 2.373 \\
        Triangular solve & 0.041 & 0.152 \\
        Total fitting wall time & 1.999 & 5.380 \\
        Relevance prediction & 223.272 & 948.733 \\
        \addlinespace[2pt]
        Total GP wall time & 225.271 & 959.180 \\
        \midrule
        Unsuccessful attempts (included above) & 0.218 & 0.000 \\
        \bottomrule
    \end{tabular*}
\end{table}

\FloatBarrier
\noindent\textbf{Agreement with a dense GP.} On a separate, duration-stratified subset of $40$ videos chosen before inspecting results, we compare predictors with identical observations, prior mean, length scale, diagonal adjustment, and FP64 arithmetic. Both predict relevance for the same unscored frames and use the same coverage term to choose the next frame.

After clipping predictions to the normalized score range, the mean absolute difference, averaged over observed sets, is $0.00706$ (95th percentile: $0.02375$). Both predictors choose the same next frame in $424$ of $451$ sets ($94.0\%$), indicating close agreement on the tested pools.

\noindent\textbf{Downstream QA with dense and banded GP prediction.} Agreement on a fixed observed set does not establish whether the approximation affects the final answer: small prediction differences can change subsequent scoring decisions and the frames supplied to the MLLM. To test this downstream effect, we compare dense-GP and banded-GP FORTE on the full evaluation sets of LongVideoBench, Video-MME, LVBench, and MLVU, using the same splits and questions as Table~\ref{tab:main-results}. The comparison uses $\rho=0.5$ and $K=64$, with BLIP and Qwen3-VL-8B-Instruct. Both variants start from identical uniform observations, then independently perform adaptive scoring, global keyframe optimization, and answer generation. Their scored pools can therefore diverge as relevance predictions change.

The dense reference uses the full kernel matrix and all observed frames in each posterior-mean prediction; the banded variant uses the approximation in Appendix~\ref{app:gp-approximation}. Both retain the same RBF kernel, prior-mean rule, length scale, FP64 arithmetic, and diagonal-adjustment schedule. Scoring and final-frame budgets, normalization and clipping, coverage weights, batching, final optimization, completion, and downstream inference settings remain unchanged. Table~\ref{tab:dense-gp-qa} reports full-benchmark QA accuracy and GP fitting-and-prediction time, measured on the same CPU host and averaged over three timing repeats.

The reported four-benchmark mean accuracies differ by $0.1$ percentage point in favor of dense GP, with absolute differences of at most $0.2$ points on individual benchmarks. Banded GP reduces fitting-and-prediction time from $15.07$ to $0.81$ seconds per question, an approximately $18.6\times$ speedup. This substantial reduction in GP computation, accompanied by a small observed accuracy difference, supports using the banded approximation as FORTE's default predictor.

\begin{table}[htbp]
    \centering
    \caption{Dense-GP and banded-GP FORTE on the full evaluation sets of all four benchmarks at $\rho=0.5$ and $K=64$, using the same questions as Table~\ref{tab:main-results}. Accuracy is reported in percent; Avg.\ is the unweighted four-benchmark mean. GP time includes fitting and prediction across all scoring rounds per question, excluding scorer and MLLM inference. Times are averaged within each benchmark and then equally across benchmarks.}
    \label{tab:dense-gp-qa}
    \small
    \setlength{\tabcolsep}{3pt}
    \begin{tabular*}{\linewidth}{@{\extracolsep{\fill}}lcccccc@{}}
        \toprule
        GP variant & LongVB & Video-MME & LVBench & MLVU & Avg. & \shortstack{GP time\\(s/question)} \\
        \midrule
        Banded (default) & 68.1 & 69.9 & 50.7 & 75.7 & 66.1 & 0.81 \\
        Dense & 68.3 & 69.8 & 50.9 & 75.8 & 66.2 & 15.07 \\
        \bottomrule
    \end{tabular*}
\end{table}

\FloatBarrier
\subsection{Computational Cost}
\label{app:online-cost}

Table~\ref{tab:mean-cost} gives the runtime breakdown underlying Figure~\ref{fig:online-core-cost} and peak GPU memory. Fixed-budget methods use $\rho=0.5$; FOCUS retains its official scoring schedule.

\noindent\textbf{Videos and repetitions.} We select 25 videos per benchmark, five from each of five equally sized duration groups, and pair each with its first annotated question. Every method runs three times on these 100 pairs using the same CPU host and A100-SXM4-80GB GPU. We average over repeats, then over pairs.

\noindent\textbf{Measurement protocol.} Each run has eight logical CPUs: decoding uses four workers with one codec thread each, preprocessing uses four threads, and math libraries use one. Frames are decoded in chunks of 32 and scored in FP32 mini-batches of 16. Selection and answering run sequentially on the same GPU; their separately measured wall times are summed, excluding loading and warmup. Inference uses CUDA-synchronized timing, and decoding records exposed wait rather than summed worker CPU time.

\noindent\textbf{Timing components.} Video decoding includes decoding for both scoring and answering. Scorer forward measures BLIP inference, while MLLM forward includes visual encoding, prompt processing, and autoregressive answer generation. FORTE's adaptive-scoring component includes GP fitting, relevance prediction, and coverage-based acquisition; its final optimization is reported separately. Other operations include preprocessing, data transfers, input construction, and remaining pipeline work. A dash denotes a component that is inapplicable or not separately reported, rather than a measured zero.

\noindent\textbf{Cost breakdown.} Decoding and scorer inference dominate runtime. FORTE spends $0.83$ seconds on adaptive scoring and $0.05$ seconds on final optimization, together $3.4\%$ of its $25.56$-second total. Its latency is within $0.7$ seconds of all compared fixed-grid baselines. Uniform is faster because it requires no relevance scoring.

\noindent\textbf{GPU memory.} Peak allocated memory is the maximum over videos and repeats, reported separately for selection and answering. All scored methods reach $2.52$\,GiB and $23.46$\,GiB, respectively; FORTE does not increase either peak. Uniform's selection-stage memory is unreported.

\begin{table}[htbp]
    \centering
    \caption{Mean runtime and peak allocated GPU memory on 100 videos with three repeats per method. Fixed-budget scored methods use $\rho=0.5$, while FOCUS follows its official schedule; all use $K=64$. Dashes denote inapplicable or separately unreported components. Rounding can cause displayed components to differ slightly from the reported totals.}
    \label{tab:mean-cost}
    \small
    \setlength{\tabcolsep}{2pt}
    \begin{tabular*}{\linewidth}{@{\extracolsep{\fill}}lrrrrrrrrr@{}}
        \toprule
        \multirow{3}{*}{Method} & \multicolumn{7}{c}{Mean time (s / question)} & \multicolumn{2}{c}{Peak GPU (GiB)} \\
        \cmidrule(lr){2-8}\cmidrule(lr){9-10}
        & \shortstack{Video\\decoding} & \shortstack{Scorer\\forward} & \shortstack{Adaptive\\scoring} & \shortstack{Final\\selection} & \shortstack{Other\\operations} & \shortstack{MLLM\\forward} & Total & Selection & Answering \\
        \midrule
        Uniform & 10.58 & --   & --   & --   & 2.11 & 2.51 & 15.20 & --   & 23.46 \\
        Top-K   & 10.67 & 7.87 & --   & --   & 3.85 & 2.58 & 24.97 & 2.52 & 23.46 \\
        AKS     & 10.61 & 7.87 & --   & --   & 3.85 & 2.54 & 24.87 & 2.52 & 23.46 \\
        FOCUS   & 10.70 & 6.93 & --   & --   & 3.67 & 2.49 & 23.79 & 2.52 & 23.46 \\
        WFS-SB  & 10.56 & 7.87 & --   & --   & 3.87 & 2.57 & 24.87 & 2.52 & 23.46 \\
        KFC     & 10.65 & 7.87 & --   & --   & 4.47 & 2.53 & 25.52 & 2.52 & 23.46 \\
        \midrule
        \textbf{Ours} & 10.60 & 7.87 & 0.83 & 0.05 & 3.65 & 2.56 & 25.56 & 2.52 & 23.46 \\
        \bottomrule
    \end{tabular*}
\end{table}

\FloatBarrier
\subsection{Duration Breakdowns}
\label{app:duration-results}

Table~\ref{tab:duration-breakdown} uses the annotated duration groups of LongVideoBench and Video-MME; overall accuracy weights groups by question count.

FORTE leads overall on both benchmarks. On LongVideoBench, it leads the 1\,h group, where its gain over Uniform is largest, and ties for best in the 15\,s group; other baselines lead the 1\,min and 10\,min groups. On Video-MME, FORTE leads the short and long groups, while BOLT leads the medium group. The strongest benefits thus occur on long videos, with gains also present in some shorter-duration groups.

\begin{table}[htbp]
    \centering
    \caption{QA accuracy (\%) by video duration on LongVideoBench and Video-MME, using Qwen3-VL-8B, BLIP, and $K=64$ input frames. Scored methods use $\rho=0.5$, except FOCUS with an effective budget of approximately $0.44$; Uniform requires no relevance scoring. Overall accuracy is the question-weighted average across duration groups. Bold marks the best results in each accuracy column, including ties.}
    \label{tab:duration-breakdown}
    \small
    \setlength{\tabcolsep}{2pt}
    \newcommand{\appendixhead}[1]{\makebox[0pt][c]{#1}}
    \begin{tabular*}{\linewidth}{@{\extracolsep{\fill}}lcccccccccc@{\hspace{9pt}}}
        \toprule
        \multirow{2}{*}{Method} & \multirow{2}{*}{$\rho$} & \multicolumn{5}{c}{LongVideoBench} & \multicolumn{4}{c}{Video-MME} \\
        \cmidrule(lr){3-7}\cmidrule(lr){8-11}
        & & \appendixhead{15\,s} & \appendixhead{1\,min} & \appendixhead{10\,min} & \appendixhead{1\,h} & \appendixhead{Overall} & \appendixhead{Short} & \appendixhead{Medium} & \appendixhead{Long} & \appendixhead{Overall} \\
        \midrule
        Uniform & 0 & 78.3 & 76.2 & 62.1 & 53.0 & 62.4 & 79.8 & 67.8 & 55.7 & 67.7 \\
        \midrule
        Top-K & 0.5 & \textbf{79.9} & \textbf{79.7} & 64.1 & 58.0 & 65.7 & 78.7 & 67.0 & 58.8 & 68.1 \\
        AKS & 0.5 & \textbf{79.9} & \textbf{79.7} & 64.8 & 56.0 & 65.1 & 78.7 & 70.1 & 58.8 & 69.0 \\
        BOLT & 0.5 & 64.6 & 70.9 & 65.5 & 60.3 & 63.9 & 78.7 & \textbf{70.4} & 59.1 & 69.4 \\
        Q-Frame & 0.5 & \textbf{79.9} & \textbf{79.7} & 65.3 & 53.9 & 64.4 & 78.7 & 66.7 & 57.4 & 67.6 \\
        AdaQ & 0.5 & \textbf{79.9} & \textbf{79.7} & 64.1 & 59.0 & 66.2 & 78.7 & 67.3 & 58.9 & 68.3 \\
        WFS-SB & 0.5 & \textbf{79.9} & \textbf{79.7} & 66.5 & 55.7 & 65.5 & 78.7 & 67.6 & 58.9 & 68.4 \\
        FOCUS & $\sim\!0.44$ & 76.2 & 72.1 & 65.0 & 56.9 & 64.1 & 75.2 & 64.1 & 56.7 & 65.3 \\
        KFC & 0.5 & \textbf{79.9} & \textbf{79.7} & \textbf{68.4} & 54.3 & 65.5 & 78.7 & 65.9 & 58.9 & 67.8 \\
        \textbf{Ours} & 0.5 & \textbf{79.9} & 76.2 & 67.2 & \textbf{62.4} & \textbf{68.1} & \textbf{81.2} & 69.1 & \textbf{59.4} & \textbf{69.9} \\
        \bottomrule
    \end{tabular*}
\end{table}

\FloatBarrier

\subsection{Scoring-Budget Regimes}
\label{app:budget-regimes}

When $B<K$, scored-pool selectors require supplementation to supply $K$ frames (Appendix~\ref{app:completion-policy}); when $B\ge K$, selection can remain within the scored pool if enough usable candidates are available. We define these groups using the integer budget in Appendix~\ref{app:implementation} and apply them to every method, including Uniform and FOCUS despite their different scoring protocols.

\noindent\textbf{Prevalence of $B<K$.} The fraction of questions in this regime decreases from $52.5\%$ at $\rho=0.1$ to $16.2\%$ at $0.5$ and $8.0\%$ at $1.0$ (Table~\ref{tab:budget-regime-counts}). At the default budget, only LongVideoBench and Video-MME contain such questions. These are budget-defined counts, not supplementation rates: completion may still be needed when $B\ge K$ if fewer than $K$ usable candidates remain.

\begin{table}[htbp]
    \centering
    \caption{Questions with a scoring budget smaller than the final frame budget, $B<K$, for $K=64$. Each entry gives the question count and its percentage within the corresponding benchmark. Total percentages use all $7{,}760$ evaluated questions as the denominator.}
    \label{tab:budget-regime-counts}
    \small
    \setlength{\tabcolsep}{6pt}
    \renewcommand{\arraystretch}{1.1}
    \begin{tabular*}{\linewidth}{@{\extracolsep{\fill}}lccc@{}}
        \toprule
        Benchmark & $\rho=0.1$ & $\rho=0.5$ & $\rho=1.0$ \\
        \midrule
        LongVideoBench & $773$ ($57.8\%$) & $361$ ($27.0\%$) & $361$ ($27.0\%$) \\
        Video-MME & $1{,}554$ ($57.6\%$) & $897$ ($33.2\%$) & $258$ ($9.6\%$) \\
        LVBench & $0$ & $0$ & $0$ \\
        MLVU & $1{,}746$ ($80.3\%$) & $0$ & $0$ \\
        \midrule
        \textbf{Total} & \textbf{4,073 (52.5\%)} & \textbf{1,258 (16.2\%)} & \textbf{619 (8.0\%)} \\
        \bottomrule
    \end{tabular*}
\end{table}

\noindent\textbf{Supplemented frame inputs.} A separate audit at $\rho=0.5$ and $K=64$ uses the approximation algorithm for final selection. For evaluated questions $\mathcal{Q}$, let $n_{\mathrm{sup}}(q)$ count supplemental frames for question $q$. Their fraction of all input frames is
\begin{equation}
    f_{\mathrm{sup}}=\frac{\sum_{q\in\mathcal{Q}}n_{\mathrm{sup}}(q)}{K|\mathcal{Q}|}.
    \label{eq:supplemented-frame-fraction}
\end{equation}
The denominator includes every evaluated question. A frame reused across questions is counted once per input; supplemental frames receive no additional relevance scoring.

Supplementation occurs only on LongVideoBench and Video-MME (Table~\ref{tab:supplemented-frame-inputs}). Pooled across questions, $8.13\%$ of input frames are supplemental and $91.87\%$ are scored. Equal benchmark weighting instead gives a supplemented fraction of $8.64\%$.

\begin{table}[htbp]
    \centering
    \caption{Composition of FORTE's final frame inputs at $\rho=0.5$ and $K=64$, using the approximation algorithm for final selection. Counts include all evaluated questions. Supplemented and scored percentages partition the final frame inputs; Pooled sums the counts across benchmarks before computing the percentages.}
    \label{tab:supplemented-frame-inputs}
    \small
    \setlength{\tabcolsep}{4pt}
    \renewcommand{\arraystretch}{1.1}
    \begin{tabular*}{\linewidth}{@{\extracolsep{\fill}}lrrr@{}}
        \toprule
        Benchmark & \shortstack{Supplemental / total\\frame inputs} & \shortstack{Supplemented\\(\%)} & \shortstack{Scored\\(\%)} \\
        \midrule
        LongVideoBench & $18{,}943\,/\,85{,}568$ & 22.14 & 77.86 \\
        Video-MME & $21{,}435\,/\,172{,}800$ & 12.40 & 87.60 \\
        LVBench & $0\,/\,99{,}136$ & 0.00 & 100.00 \\
        MLVU & $0\,/\,139{,}136$ & 0.00 & 100.00 \\
        \midrule
        Pooled & $40{,}378\,/\,496{,}640$ & 8.13 & 91.87 \\
        \bottomrule
    \end{tabular*}
\end{table}

\noindent\textbf{Accuracy within each regime.} Within $B\ge K$, FORTE achieves the highest observed four-benchmark mean accuracy at both scoring budgets (Table~\ref{tab:budget-regime-accuracy}). At $\rho=0.1$, it leads LongVideoBench, Video-MME, and LVBench, while KFC leads MLVU; at $\rho=0.5$, it leads all four benchmarks. Thus, the observed mean advantage is not confined to questions whose scoring budget is too small to supply $K$ input frames.

Results within $B<K$ are more mixed. FORTE has the highest observed accuracy on MLVU at $\rho=0.1$ and Video-MME at $\rho=0.5$, while Top-K, AKS, WFS-SB, and KFC tie for the highest accuracy on LongVideoBench at both budgets and Video-MME at $\rho=0.1$. For the budget-matched score-based methods, final inputs necessarily include unscored supplemental frames, so accuracy reflects scored-pool construction together with the completion rule. Because group membership differs across regimes and scoring budgets, differences in their absolute accuracies do not isolate the effect of supplementation or additional scoring.

\begin{table}[!htb]
    \centering
    \caption{QA accuracy (\%) by scoring-budget regime at $\rho\in\{0.1,0.5\}$ and $K=64$, using Qwen3-VL-8B and BLIP. FORTE uses the exact algorithm for final optimization. Within each budget block, all methods share identical question groups defined by the corresponding integer scoring budget. Uniform requires no scoring. $\dagger$FOCUS is reported only in the $\rho=0.5$ block at its effective allowance of approximately $0.44$ and is not budget-matched. Avg.\ is the unweighted four-benchmark mean within $B\ge K$. The $B<K$ group contains no LVBench questions at either budget and no MLVU questions at $\rho=0.5$, so no four-benchmark mean is reported for it. Bold marks the highest observed accuracy in each row, including ties.}
    \label{tab:budget-regime-accuracy}
    \small
    \setlength{\tabcolsep}{2pt}
    \renewcommand{\arraystretch}{1.1}
    \begin{tabular*}{\linewidth}{@{\extracolsep{\fill}}llr*{7}{c}@{}}
        \toprule
        Regime & Benchmark & Questions & FORTE & Uniform & Top-K & AKS & WFS-SB & KFC & FOCUS$^{\dagger}$ \\
        \midrule
        \multicolumn{10}{@{}l}{\emph{Scoring budget $\rho=0.1$}} \\
        \midrule
        \multirow{3}{*}{$B<K$} & LongVB & $773$ & 70.2 & 69.2 & \textbf{70.4} & \textbf{70.4} & \textbf{70.4} & \textbf{70.4} & -- \\
        & Video-MME & $1{,}554$ & 74.5 & 74.7 & \textbf{75.2} & \textbf{75.2} & \textbf{75.2} & \textbf{75.2} & -- \\
        & MLVU & $1{,}746$ & \textbf{73.3} & 71.9 & 71.8 & 71.8 & 71.8 & 71.8 & -- \\
        \midrule
        \multirow{5}{*}{$B\ge K$} & LongVB & $564$ & \textbf{57.6} & 53.0 & 53.9 & 53.4 & 53.4 & 51.6 & -- \\
        & Video-MME & $1{,}146$ & \textbf{62.0} & 58.3 & 60.0 & 60.9 & 59.0 & 58.8 & -- \\
        & LVBench & $1{,}549$ & \textbf{49.0} & 41.0 & 46.6 & 45.6 & 46.9 & 47.6 & -- \\
        & MLVU & $428$ & 68.7 & 63.8 & 66.9 & 67.7 & 68.8 & \textbf{69.4} & -- \\
        \cmidrule(l){2-10}
        & Avg. & -- & \textbf{59.3} & 54.0 & 56.9 & 56.9 & 57.0 & 56.9 & -- \\
        \midrule
        \multicolumn{10}{@{}l}{\emph{Scoring budget $\rho=0.5$}} \\
        \midrule
        \multirow{2}{*}{$B<K$} & LongVB & $361$ & 78.1 & 77.3 & \textbf{79.8} & \textbf{79.8} & \textbf{79.8} & \textbf{79.8} & 74.2 \\
        & Video-MME & $897$ & \textbf{81.3} & 79.8 & 78.7 & 78.7 & 78.7 & 78.7 & 75.4 \\
        \midrule
        \multirow{5}{*}{$B\ge K$} & LongVB & $976$ & \textbf{64.4} & 56.9 & 60.5 & 59.7 & 60.2 & 60.2 & 60.3 \\
        & Video-MME & $1{,}803$ & \textbf{64.3} & 61.7 & 62.8 & 64.2 & 63.2 & 62.4 & 60.3 \\
        & LVBench & $1{,}549$ & \textbf{50.7} & 41.0 & 50.6 & 47.8 & 50.6 & 46.4 & 48.7 \\
        & MLVU & $2{,}174$ & \textbf{75.7} & 70.3 & 73.8 & 74.5 & 74.7 & 74.9 & 71.3 \\
        \cmidrule(l){2-10}
        & Avg. & -- & \textbf{63.8} & 57.5 & 61.9 & 61.5 & 62.2 & 61.0 & 60.2 \\
        \bottomrule
    \end{tabular*}
\end{table}

\subsection{Further Mechanism Controls}
\label{app:pending-mechanisms}

Each variant in Table~\ref{tab:pending-mechanisms} changes only the named component of default FORTE, with variants grouped by stage.

\noindent\textbf{Stage 1: Adaptive relevance scoring.} Replacing GP prediction with linear interpolation or the nearest observed score reduces accuracy on all four benchmarks.

The zero-mean GP sets $\mu_{\mathrm{obs}}=0$ in Equation~\ref{eq:gp-mean} instead of the observed average. Centering improves every benchmark and raises mean accuracy from $65.7\%$ to $66.1\%$, supporting a nonzero baseline in regions without nearby observations.

The round-frozen variant holds coverage gains fixed throughout each batch rather than reducing them around newly chosen frames. Its lower accuracy supports frame-by-frame coverage updates even when relevance predictions remain fixed.

\noindent\textbf{Stage 2: Global keyframe optimization.} Both controls keep the scored pool and measured relevance fixed. The square-root variant replaces the logarithmic gap reward with a scaled square-root reward, matching the marginal gain for midpoint insertion into a gap of length $\tau_K$ without changing the coverage weight. The other variant removes boundary-gap rewards, retaining only gaps between selected frames. Both reduce accuracy on all four benchmarks, supporting the chosen reward and boundary treatment.

\begin{table}[htbp]
    \centering
    \caption{Controls for the two stages of FORTE at $\rho=0.5$ and $K=64$. Ours uses the centered GP with the observed average as its prior mean. Variants are grouped by the stage they modify, with all other settings fixed. Values are accuracy (\%).}
    \label{tab:pending-mechanisms}
    \small
    \setlength{\tabcolsep}{2pt}
    \newcommand{\appendixhead}[1]{\makebox[0pt][c]{#1}}
    \begin{tabular*}{\linewidth}{@{\extracolsep{\fill}}lccccc@{\hspace{2pt}}}
        \toprule
        Variant & \appendixhead{LongVB} & \appendixhead{Video-MME} & \appendixhead{LVBench} & \appendixhead{MLVU} & \appendixhead{Avg.} \\
        \midrule
        \textbf{Ours (centered GP)} & \textbf{68.1} & \textbf{69.9} & \textbf{50.7} & \textbf{75.7} & \textbf{66.1} \\
        \midrule
        \multicolumn{6}{@{}l}{\emph{Stage 1: Adaptive relevance scoring}} \\
        Linear interpolation & 66.3 & 68.8 & 50.2 & 74.8 & 65.0 \\
        Nearest observed score & 66.3 & 68.5 & 49.6 & 74.7 & 64.8 \\
        Zero-mean GP & 67.9 & 69.8 & 50.0 & 75.2 & 65.7 \\
        Round-frozen coverage & 65.1 & 69.0 & 49.8 & 74.6 & 64.6 \\
        \midrule
        \multicolumn{6}{@{}l}{\emph{Stage 2: Global keyframe optimization}} \\
        Matched square-root potential & 66.9 & 68.5 & 49.5 & 74.9 & 64.9 \\
        No boundary-gap rewards & 67.4 & 68.6 & 49.9 & 74.2 & 65.0 \\
        \bottomrule
    \end{tabular*}
\end{table}

\FloatBarrier

\subsection{Comparison with Smooth Relevance Predictors}
\label{app:smooth-predictors}

Table~\ref{tab:smooth-predictors} replaces the GP mean predictor at $\rho\in\{0.1,0.5\}$ while keeping initialization, batching, coverage updates and weights, normalization, completion, and final optimization fixed. Variants share initial observations but subsequently acquire frames according to their own predictions.

The primary control is normalized RBF kernel regression, which predicts relevance by averaging observed scores with temporal kernel weights:
\begin{equation}
    \mu_{\mathrm{RBF}}(t)=\frac{\sum_{t_i\in\mathcal{O}}k(t,t_i)r_i}{\sum_{t_i\in\mathcal{O}}k(t,t_i)}.
    \label{eq:rbf-regression-control}
\end{equation}
Here, $k$ is the kernel in Equation~\ref{eq:rbf-kernel} and $r_i$ the normalized observed score. RBF regression requires no kernel-system solve and shares the GP's fixed $\ell=10$ seconds, without predictor-specific tuning. We also evaluate shape-preserving piecewise cubic Hermite interpolation (PCHIP), complementing the simpler controls in Table~\ref{tab:pending-mechanisms}.

\noindent\textbf{Accuracy and cost.} The GP leads every benchmark at both budgets, improving mean accuracy over RBF regression by $0.7$ points at $\rho=0.5$ and $1.1$ points at $\rho=0.1$. The larger gain with fewer observations suggests that prediction quality matters more under sparse scoring. RBF regression and PCHIP are faster: the GP adds $0.28$--$0.41$ seconds over RBF regression while remaining below one second of fitting-and-prediction time per question.

\begin{table}[htbp]
    \centering
    \caption{Comparison of relevance predictors at two scoring budgets, with $K=64$ and global keyframe optimization. Only the predictor used during adaptive scoring changes. GP prediction and RBF regression share the same RBF kernel and fixed length scale, $\ell=10$ seconds. Accuracy is reported in percent; Avg.\ is the unweighted four-benchmark mean. Time is the mean total fitting-and-prediction time across all scoring rounds per video--question pair, excluding scorer and MLLM inference.}
    \label{tab:smooth-predictors}
    \small
    \setlength{\tabcolsep}{3pt}
    \begin{tabular*}{\linewidth}{@{\extracolsep{\fill}}lcccccc@{}}
        \toprule
        Predictor & LongVB & Video-MME & LVBench & MLVU & Avg. & Time (s) \\
        \midrule
        \multicolumn{7}{@{}l}{\emph{Scoring budget $\rho=0.5$}} \\
        FORTE (GP mean) & 68.1 & 69.9 & 50.7 & 75.7 & 66.1 & 0.81 \\
        RBF regression & 67.0 & 69.3 & 50.1 & 75.1 & 65.4 & 0.40 \\
        PCHIP interpolation & 66.8 & 69.0 & 50.3 & 75.0 & 65.3 & 0.21 \\
        \midrule
        \multicolumn{7}{@{}l}{\emph{Scoring budget $\rho=0.1$}} \\
        FORTE (GP mean) & 64.9 & 69.2 & 49.0 & 72.0 & 63.8 & 0.59 \\
        RBF regression & 63.7 & 68.8 & 47.4 & 71.0 & 62.7 & 0.31 \\
        PCHIP interpolation & 63.8 & 68.9 & 47.5 & 71.4 & 62.9 & 0.18 \\
        \bottomrule
    \end{tabular*}
\end{table}

\FloatBarrier

\subsection{Posterior Uncertainty versus Temporal Coverage}
\label{app:uncertainty-acquisition}

\noindent\textbf{Relation to GP-based sequential acquisition.} The adaptive-scoring stage is related to Gaussian-process-based Bayesian optimization and sequential experimental design, which use posterior predictions to choose subsequent evaluations. A common strategy, including GP-UCB-style acquisition, promotes exploration through posterior uncertainty. FORTE instead uses the GP posterior mean to estimate frame relevance and introduces exploration explicitly through marginal temporal coverage. The comparison below isolates these choices under the same predictor and scoring budget.

Table~\ref{tab:uncertainty-acquisition} tests posterior uncertainty as an alternative or addition to coverage at $\rho\in\{0.1,0.5\}$. The four rules use GP mean only, mean plus posterior standard deviation (GP-UCB-style acquisition), mean plus coverage (FORTE), or all three terms. All share $\mathcal{U}$, $B$, initial observations, the GP mean predictor, four-round schedule, normalization, completion, and exact final optimizer; only acquisition changes.

\noindent\textbf{Uncertainty-aware acquisition.} The uncertainty bonus is $\lambda\sigma(t)$, where $\sigma(t)$ is the GP posterior standard deviation of relevance at an unscored frame. Before batching or numerical approximation, its reference definition is
\begin{equation}
    \sigma^2(t)=k(t,t)-\mathbf k(t)^\top(\mathbf K+\epsilon\mathbf{I})^{-1}\mathbf k(t).
    \label{eq:gp-variance-control}
\end{equation}
We retain FORTE's unit-amplitude RBF kernel and $\ell=10$ seconds, set $\lambda=1$, and use $\beta_{\mathrm{acq}}=3.5$ and $\tau_B=T/(B+1)$ when coverage is included. Within each batch, the mean remains fixed while uncertainty is conditioned on each newly chosen frame's location, without requiring its unobserved score. Coverage also updates after every choice. Both exploration mechanisms can therefore discourage redundant within-batch choices.

\begin{table}[htbp]
    \centering
    \caption{Exploration mechanisms for adaptive relevance scoring, with $K=64$ and the same exact final optimizer. Uncertainty denotes $\lambda\sigma(t)$ with $\lambda=1$; coverage denotes $\beta_{\mathrm{acq}}\Delta C_{\tau_B}$ with $\beta_{\mathrm{acq}}=3.5$. Accuracy is reported in percent; Avg.\ weights the four benchmarks equally. Bold marks the highest accuracy in each budget block. Acquisition time includes GP fitting, prediction, and batch construction, excluding video decoding, scorer inference, final optimization, and MLLM inference. Times are per video--question pair, averaged within benchmarks and then equally across benchmarks.}
    \label{tab:uncertainty-acquisition}
    \small
    \setlength{\tabcolsep}{2pt}
    \begin{tabular*}{\linewidth}{@{\extracolsep{\fill}}lrrrrrr@{}}
        \toprule
        Acquisition rule & LongVB & Video-MME & LVBench & MLVU & Avg. & Time (s) \\
        \midrule
        \multicolumn{7}{@{}l}{\emph{Scoring budget $\rho=0.1$}} \\
        Mean only & 63.1 & 68.0 & 47.5 & 71.4 & 62.5 & 0.60 \\
        Mean + uncertainty & 63.6 & 68.4 & 47.8 & 71.8 & 62.9 & 0.95 \\
        Mean + coverage (FORTE) & \textbf{64.9} & \textbf{69.2} & 49.0 & 72.0 & \textbf{63.8} & 0.65 \\
        Mean + uncertainty + coverage & 64.8 & 69.1 & \textbf{49.1} & \textbf{72.1} & \textbf{63.8} & 1.00 \\
        \midrule
        \multicolumn{7}{@{}l}{\emph{Scoring budget $\rho=0.5$}} \\
        Mean only & 66.4 & 67.9 & 49.6 & 73.0 & 64.2 & 0.80 \\
        Mean + uncertainty & 67.0 & 68.4 & 49.8 & 73.6 & 64.7 & 1.30 \\
        Mean + coverage (FORTE) & \textbf{68.1} & \textbf{69.9} & \textbf{50.7} & \textbf{75.7} & \textbf{66.1} & 0.83 \\
        Mean + uncertainty + coverage & 68.0 & 69.8 & \textbf{50.7} & 75.6 & 66.0 & 1.33 \\
        \bottomrule
    \end{tabular*}
\end{table}

\noindent\textbf{Accuracy.} Uncertainty improves on mean-only acquisition on every benchmark, but coverage yields larger gains at both budgets. Adding uncertainty to coverage leaves mean accuracy unchanged at $\rho=0.1$ and slightly lower at $0.5$. Thus, at the tested setting, uncertainty helps without coverage but adds no mean accuracy benefit once coverage is present.

\noindent\textbf{Acquisition cost.} Adding uncertainty to FORTE increases acquisition time from $0.65$ to $1.00$ seconds at $\rho=0.1$ and from $0.83$ to $1.33$ seconds at $0.5$, without improving mean accuracy. Coverage alone therefore offers the better observed accuracy--cost trade-off.

\FloatBarrier

\section{Extended Analysis of Related Work}
\label{app:related-work}

This section supplements Section~\ref{sec:related-work} with a broader overview of long-video modeling and keyframe selection, focusing on their connections to adaptive scoring and global keyframe optimization.

\subsection{Video LLMs for Long-Video Understanding}

Video LLMs such as LLaVA-OneVision, LLaVA-Video, Qwen3-VL, and InternVL3.5 combine visual representations with language reasoning \citep{llava-onevision,llava-video,qwen3vl,internvl35}. To extend these capabilities to long videos, LongVILA scales the supported context \citep{longvila}, while MovieChat, MA-LMM, and VideoStreaming retain historical information through memory mechanisms \citep{moviechat,ma-llm,videostreaming}. LongVLM, LongVU, Video-XL, and VideoChat-Flash use hierarchical representations or visual compression \citep{longvlm,longvu,video-xl,videochat-flash}; FLoC further explores facility-location-based token compression \citep{floc}. These approaches improve how video information is represented and retained. Keyframe selection complements them by determining which frames form the model's visual input.

\subsection{Keyframe Selection Methods}

\noindent\textbf{Training-based selection.} SeViLA, Frame-Voyager, and ReFocus learn selection policies from supervision or feedback \citep{sevila,frame-voyager,refocus}. TimeSearch-R, VideoZoomer, and LongVideo-R1 use reinforcement learning to learn temporal search policies \citep{timesearch-r,videozoomer,longvideo-r1}. In contrast, FORTE operates with frozen models and requires no additional selector training.

\noindent\textbf{LLM-guided evidence search.} VideoAgent coordinates vision--language tools for iterative retrieval \citep{videoagent}, while VideoTree and DrVideo organize captioned or text-based evidence for question answering \citep{videotree,drvideo}. A.I.R. and DIG guide evidence selection through VLM analysis or MLLM feedback \citep{air,dig}. Such workflows use semantic judgments beyond frame--query similarity, but their additional captioning, analysis, and reasoning calls can be computationally expensive. FORTE instead uses a lightweight relevance scorer without invoking the downstream answer model during selection.

\subsubsection{Score-Based Frame Selection}

\noindent\textbf{Relevance-based selection.} Pretrained vision--language models such as CLIP and BLIP provide frame--query relevance scores without task-specific training \citep{clip,blip}. KeyVideoLLM uses text--frame matching to select relevant frames \citep{keyvideollm}. Beyond Top-K selection, BOLT samples according to temporal relevance mass, Q-Frame combines relevance-based sampling with resolution adaptation, and AdaQ adjusts its sampling distribution using score statistics \citep{bolt,q-frame,adaq}. These methods refine the selection of input frames from scored candidates; when that candidate pool is fixed in advance, improving the final selection rule cannot recover evidence omitted before scoring.

AKS balances relevance with temporal coverage, while TCS combines relevance-focused clips with sparse global context \citep{aks,tcs}. WFS-SB~\citep{wfssb} detects semantic boundaries in the relevance trajectory and selects diverse frames within the resulting segments. Nar-KFC incorporates visual diversity \citep{narkfc}. Together, these methods highlight the importance of selecting complementary evidence rather than relying on relevance alone. FOCUS allocates relevance evaluations using clip-level mean relevance and confidence bounds \citep{focus}.

FORTE jointly addresses where to spend a limited scoring budget and which scored frames to retain. Its adaptive scoring stage combines frame-level relevance predictions with marginal temporal coverage to construct the candidate pool. Its global optimization stage then balances measured relevance and temporal coverage through a logarithmic gap objective, for which we derive an exact $\mathcal{O}(KN)$ algorithm (Section~\ref{sec:final-selection}). Our contribution lies in coupling adaptive candidate acquisition with efficient, exact optimization of the final input.


\FloatBarrier

\section{Qualitative Analysis}
\label{app:qualitative}

We compare Uniform, Top-K, AKS, FOCUS, and FORTE using their selected frames and saved answers from Qwen3-VL-8B-Instruct. All methods supply $K=64$ frames; Top-K, AKS, and FORTE use BLIP with a nominal scoring budget of $\rho=0.5$, while FOCUS follows its official scoring procedure. FORTE uses the exact algorithm for global keyframe optimization. The following cases distinguish acquiring evidence in the scored pool from retaining it in the final input, and include a shared failure despite retained evidence. They are selected examples rather than an estimate of how frequently each behavior occurs.

\begin{figure}[p]
    \centering
    \includegraphics[width=\linewidth]{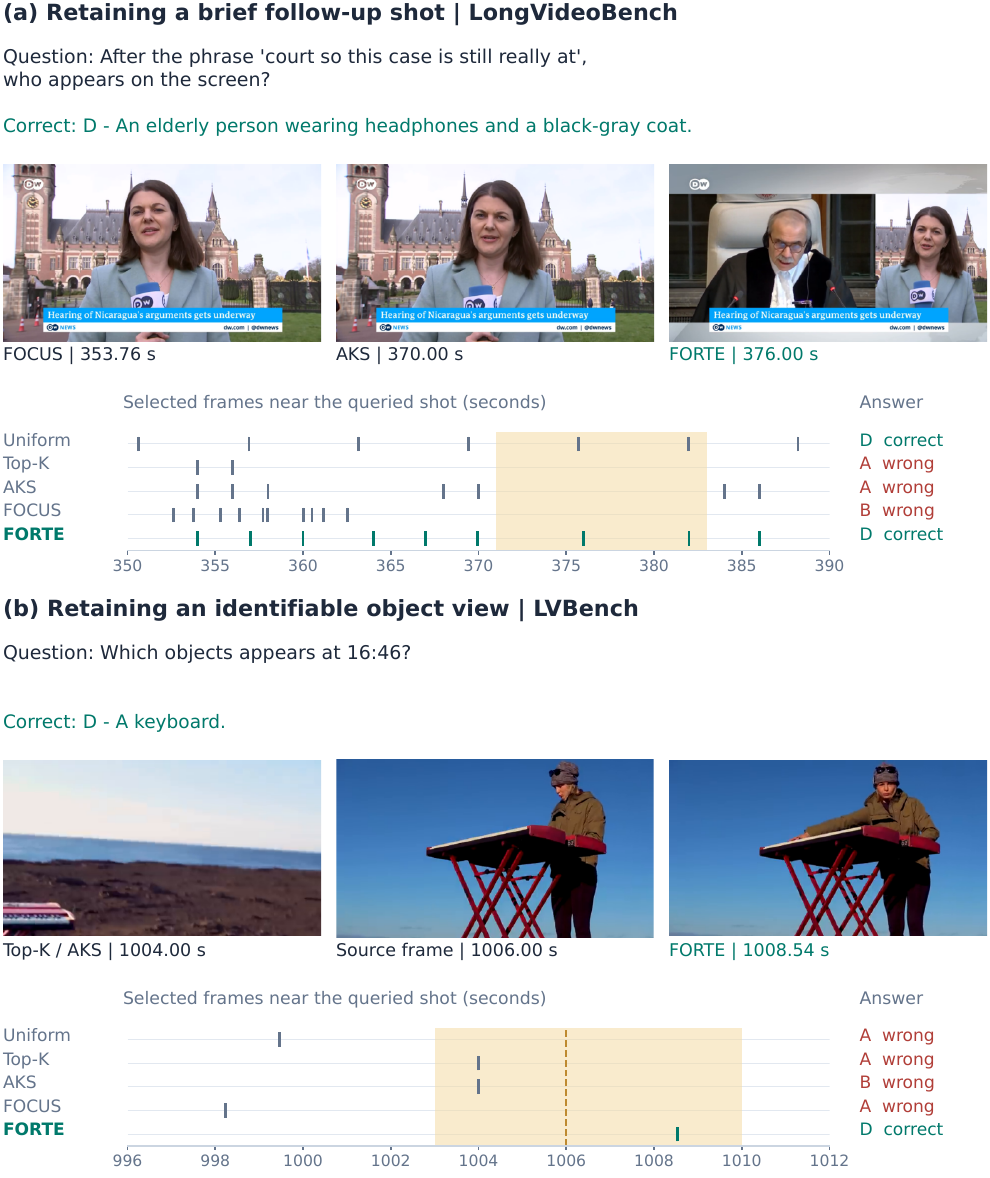}
    \caption{Brief evidence may be omitted from the final input or shown only partially. Each tick marks a selected frame within the displayed time range; shading highlights the queried interval or its surrounding shot. (a) FORTE and Uniform retain the courtroom follow-up at 371--383 seconds. (b) FORTE retains a clear keyboard view near the requested timestamp (dashed line). The middle image in (b) is a source-video reference, not a selected input; all other images are selected-frame excerpts.}
    \label{fig:qualitative-localized}
\end{figure}

\begin{figure}[p]
    \ContinuedFloat
    \centering
    \includegraphics[width=\linewidth]{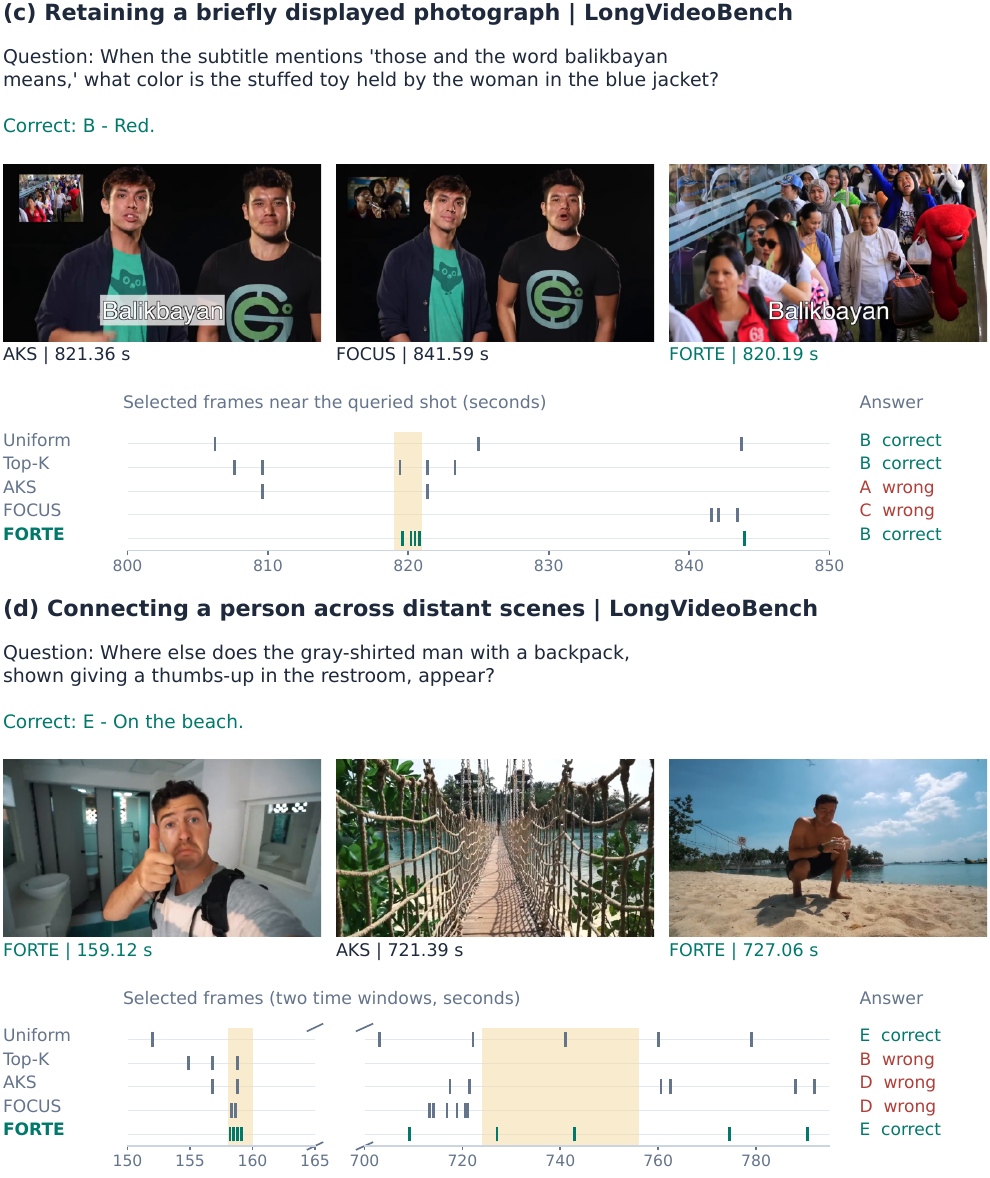}
    \caption{Additional baseline comparisons (continued). (c) The full-size photograph makes the toy's color readily visible; Top-K also retains this view and answers correctly. (d) FORTE retains the restroom reference and the later shore sequence. The broken timeline shows two separate time windows with different scales; shading marks the annotated intervals. All images are selected-frame excerpts, and all ticks represent selected frames within the displayed windows.}
    \label{fig:qualitative-additional}
\end{figure}

\begin{figure}[p]
    \ContinuedFloat
    \centering
    \includegraphics[width=\linewidth]{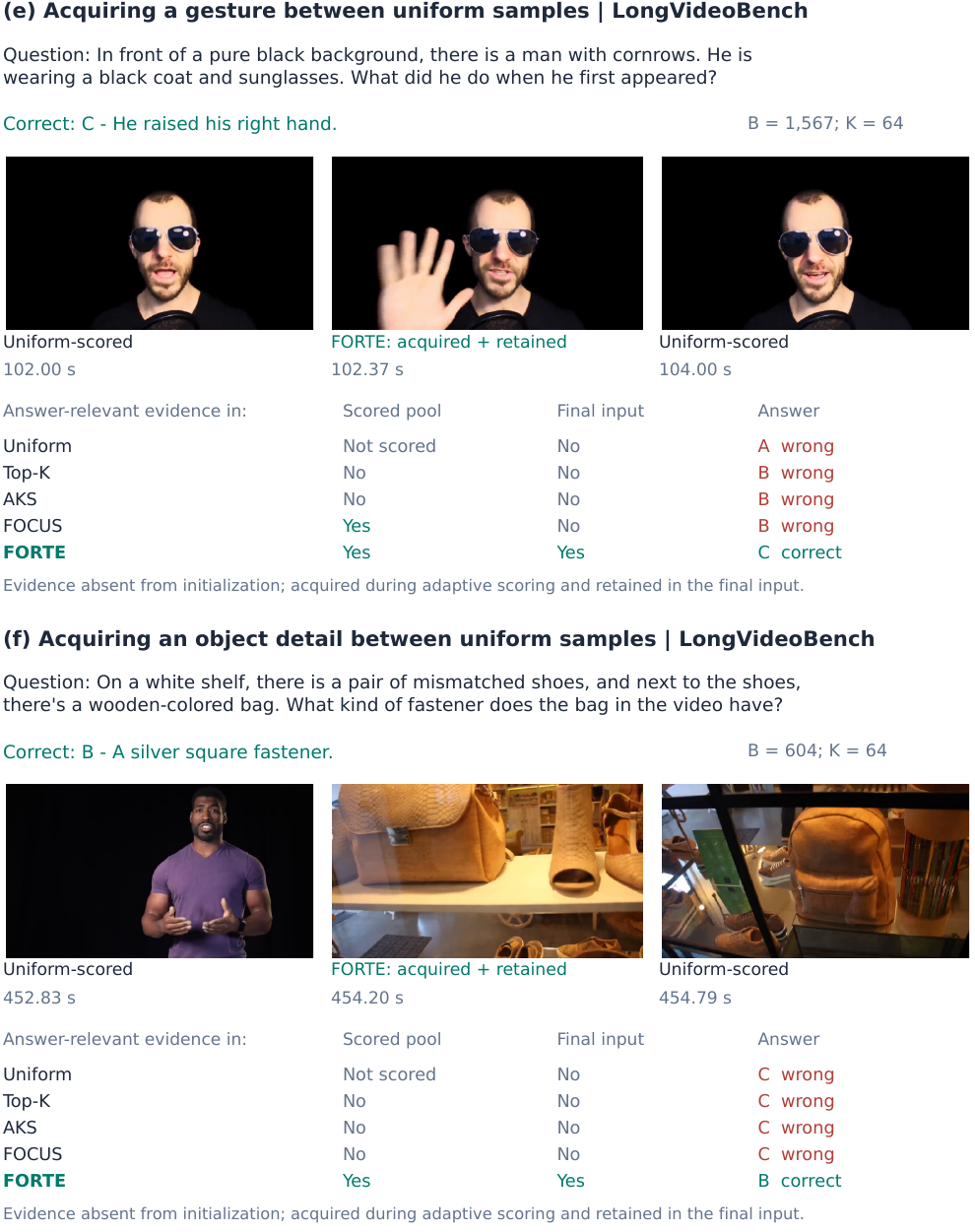}
    \caption{Evidence acquisition and final retention (continued). FORTE captures (e) a brief gesture and (f) a bag's fastener between uniform-scored frames, and is the only method to answer correctly in both cases. The tables distinguish visually usable evidence in the scored pool from evidence retained in the final input, alongside each method's answer. Uniform selects its final input without scoring. Top-K, AKS, and FORTE each score 1,567 frames in (e) and 604 in (f); FOCUS scores 1,369 and 509, respectively. All methods retain 64 frames without supplementation. Uniform-scored images are source-video frames; FORTE images are reconstructed selected inputs.}
    \label{fig:qualitative-acquisition}
\end{figure}

\paragraph{Preserving brief evidence.} Figure~\ref{fig:qualitative-localized}(a) asks who appears after a reporter's spoken cue. FORTE retains the subsequent courtroom shot, whereas Top-K, AKS, and FOCUS omit the specified follow-up interval. Importantly, Top-K and AKS had scored frames from this interval but did not retain them: making evidence available in the candidate pool does not ensure that it survives final selection. FOCUS instead retains several similar views of the reporter before the transition. FORTE and Uniform both preserve the follow-up and answer correctly, illustrating the value of maintaining temporal coverage when the most visually prominent subject is not the answer.

Figure~\ref{fig:qualitative-localized}(b) illustrates a related distinction between temporal proximity and readable evidence. The question asks which object appears at 16:46 in a roughly 42-minute video. FORTE retains a nearby view showing the keyboard and its stand, and is the only method to answer correctly. Top-K and AKS retain an earlier frame in which only an edge of the instrument is visible, while Uniform and FOCUS omit the shot. Thus, selecting a frame near the requested time is not equivalent to retaining a view from which the object can be identified.

\paragraph{Retaining a readable view.} Figure~\ref{fig:qualitative-additional}(c) asks for the color of a stuffed toy in a briefly displayed photograph. FORTE and Top-K retain the full-size photograph, in which the red toy is readily visible, and both answer correctly. AKS retains a nearby view of the presenters with a smaller inset of the photograph, while FOCUS skips the full-size image; both answer incorrectly. Uniform also answers correctly despite having no frame in the narrow full-size-photo interval. The comparison illustrates a difference in the visibility of the queried detail, rather than establishing that this particular view is necessary for a correct answer.

\paragraph{Linking distant scenes.} Figure~\ref{fig:qualitative-additional}(d) requires matching a man identified in a restroom to another location in the video. FORTE retains both the reference appearance and direct views of the man on the beach more than nine minutes later. Top-K, AKS, and FOCUS retain the restroom reference but omit the annotated shore sequence. AKS and FOCUS instead retain nearby suspension-bridge views and answer ``suspension bridge,'' whereas FORTE and Uniform answer ``beach.'' AKS and FOCUS also retain later beach-adjacent scenes, so the distinction is not the complete absence of beach imagery: FORTE preserves the direct matching scene together with the reference appearance.

\paragraph{Acquiring evidence absent from the scored pool.} The preceding comparisons primarily concern which evidence survives final selection. Figure~\ref{fig:qualitative-acquisition}(e--f) instead examines evidence that uniform scoring never observes. We inspect the scored pools as well as the final inputs, checking whether a frame shows the queried gesture or object detail rather than merely falling within the annotated interval. Top-K, AKS, and FORTE have matched scoring counts in each case, with $B>K$; all five methods return 64 frames without supplementation.

\begin{samepage}
In Figure~\ref{fig:qualitative-acquisition}(e), the question concerns the man's first appearance in a 52-minute video. Uniform scoring captures frames at 102 and 104 seconds, but neither shows his brief right-hand raise between them. FORTE's initialization also misses the gesture; adaptive scoring captures it at 102.37 seconds, and final optimization retains this frame. FOCUS scores a view of the same gesture but discards it, instead retaining a later clenched fist. FORTE is the only method to answer correctly: Uniform answers that the man takes off his sunglasses, while Top-K, AKS, and FOCUS answer that he clenches his right hand into a fist. The comparison thus separates two mechanisms: FORTE acquires an event absent from the uniformly scored pool and preserves evidence that FOCUS loses during final selection.

\end{samepage}

Figure~\ref{fig:qualitative-acquisition}(f) concerns the fastener on a bag displayed beside a pair of shoes. Uniform scoring captures the presenter at 452.83 seconds and a different, zippered backpack at 454.79 seconds, missing the intervening view of the queried fastener. Although the latter frame falls within the annotated interval, it does not show the required detail. FORTE's initialization also misses this evidence; adaptive scoring acquires a clear view at 454.20 seconds, which final optimization retains. FORTE correctly identifies a silver square fastener, whereas Uniform, Top-K, AKS, and FOCUS all answer ``silver round fastener.'' The usable detail is absent from the uniformly scored pool and FOCUS's scored pool, as well as all four baselines' final inputs. Together, these cases illustrate evidence acquisition between uniform samples within the video: final optimization cannot recover visual evidence that was never included in the scored pool.

\paragraph{A shared failure despite retained evidence.} Figure~\ref{fig:qualitative-failure} asks about the woman's wedding dress. All five methods retain views showing her black dress, yet every answer is white. Several attendants wear white, making it important to distinguish the queried woman from nearby people. The failure therefore cannot be explained solely by missing views of the relevant clothing. It highlights a remaining challenge beyond evidence retention: the downstream model must associate the question with the correct visual subject and use that evidence to answer. The saved responses alone do not establish whether subject confusion or another aspect of the model's reasoning caused the error. Together with the successful cases, this example distinguishes the quality of the selected evidence from the model's ability to interpret it.

\begin{figure}[htbp]
    \centering
    \includegraphics[width=\linewidth]{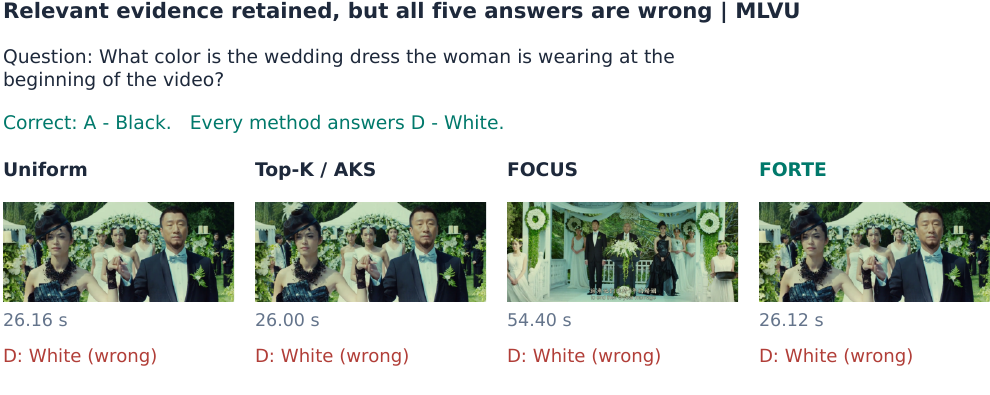}
    \caption{A shared failure on MLVU. The selected inputs visibly contain the black dress, but every method produces the incorrect answer ``white.'' Top-K and AKS select the same displayed frame. Retaining relevant evidence does not by itself guarantee that the downstream MLLM uses it correctly.}
    \label{fig:qualitative-failure}
\end{figure}

\FloatBarrier

\end{document}